\documentclass{article}

\usepackage{microtype}
\usepackage{graphicx}
\usepackage{subfigure}
\usepackage{booktabs} % for professional tables
\graphicspath{{./figs/}}
\usepackage{hyperref}

\usepackage[margin=1in]{geometry}
\usepackage{algorithm}
\usepackage{algorithmic}

\usepackage{amsmath}
\usepackage{amssymb}
\usepackage{mathtools}
\usepackage{amsthm}
\newcommand\numberthis{\addtocounter{equation}{1}\tag{\theequation}}

\newcommand{\bU}{\boldsymbol{U}}
\newcommand{\bI}{\boldsymbol{I}}
\newcommand{\bV}{\boldsymbol{V}}
\newcommand{\bW}{\boldsymbol{W}}
\newcommand{\bX}{\boldsymbol{X}}
\newcommand{\bY}{\boldsymbol{Y}}
\newcommand{\bS}{\boldsymbol{S}}
\newcommand{\bz}{\boldsymbol{z}}

\newcommand{\pt}[1]{ {\frac{\partial}{\partial {#1}} } }
\newcommand{\pst}{\frac{\partial}{\partial \sigma_t}}
\DeclareMathOperator*{\argmax}{arg\,max}
\DeclareMathOperator*{\argmin}{arg\,min}
\usepackage{enumitem}
\usepackage{nicefrac}
\newcommand{\vct}[1]{\boldsymbol{#1}} % vector
\newcommand{\mat}[1]{\boldsymbol{#1}} % matrix

\newcommand{\calR}{\mathcal{R}}

\newcommand{\bg}{\boldsymbol{g}}

\newcommand{\bx}{\boldsymbol{x}}

\usepackage[capitalize,noabbrev]{cleveref}

\theoremstyle{plain}
\newtheorem{theorem}{Theorem}[section]

\newtheorem{lemma}[theorem]{Lemma}
\newtheorem{corollary}[theorem]{Corollary}
\theoremstyle{definition}

\newtheorem{fact}[theorem]{Fact}
\newtheorem{assumption}[theorem]{Assumption}
\theoremstyle{remark}
\newtheorem{remark}[theorem]{Remark}
\usepackage[textsize=tiny]{todonotes}
\title{Hierarchical Bayes Approach to Personalized Federated Unsupervised Learning}
\author{%
	Kaan Ozkara \\
University of California, Los Angeles\\
\texttt{kaan@ucla.edu} \\
	\and
		Bruce Huang \\
	University of California, Los Angeles\\
	\texttt{brucehuang@ucla.edu} \\
	\and
		Ruida Zhou \\
	University of California, Los Angeles\\
	\texttt{ruida@ucla.edu} \\
	\and
		Suhas Diggavi \\
	University of California, Los Angeles\\
	\texttt{suhas@ee.ucla.edu} \\
}

\begin{document}
	
	\date{} %Date omission only works if it comes before
        \maketitle

%% ABSTRACT
\begin{abstract}
Statistical heterogeneity of clients' local data is an important characteristic in federated learning, motivating personalized algorithms tailored to the local data statistics. Though there has been a plethora of algorithms proposed for personalized supervised learning, discovering the structure of local data through personalized unsupervised learning is less explored. We initiate a systematic study of such personalized unsupervised learning by developing algorithms based on optimization criteria inspired by a hierarchical Bayesian statistical framework. We develop adaptive algorithms that discover the balance between using limited local data and collaborative information. We do this in the context of two unsupervised learning tasks: personalized dimensionality reduction and personalized diffusion models. We develop convergence analyses for our adaptive algorithms which illustrate the dependence on problem parameters (\emph{e.g.,} heterogeneity, local sample size). We also develop a theoretical framework for personalized diffusion models, which shows the benefits of collaboration even under heterogeneity. We finally evaluate our proposed algorithms using synthetic and real data, demonstrating the effective sample amplification for personalized tasks, induced through collaboration, despite data heterogeneity.
\end{abstract}

%% INTRO
\section{Introduction}
\label{sec:intro}
One of the goals of unsupervised learning is to discover the underlying structure in data and use this for tasks such as dimensionality reduction and generating new samples from the data distribution. We might want to perform this task on local data of a client, \emph{e.g.,} data collected from personal sensors or other devices; such data could have heterogeneous statistics across clients.  The desired \emph{personalized} task should be tailored to the particular distribution of the local data, and hence to discover this structure one might need a significant amount of local data. There might be insufficient local samples for the task, motivating collaboration between clients. Moreover, as argued in the federated learning (FL) paradigm, we would like to leverage data across clients without explicit data sharing \cite{mcmahan2017communicationefficient,kairouz2021advances}. In this paper, we initiate a systematic study of personalized federated unsupervised learning, where clients collaborate to discover personalized structure in their (heterogeneous) local data.

There has been a plethora of personalized learning models proposed in the literature, mostly for \emph{supervised} federated learning \cite{fallah2020personalized,dinh2020personalized,mansour2020approaches,ozkara2021quped,ozkara2023}. These methods were motivated by the statistical heterogeneity of local data, causing a single ``global'' model to perform poorly on local data. The different personalized federated supervised learning algorithms were unified in \cite{ozkara2023} using an empirical/hierarchical Bayes statistical model \cite{efron_2010}, which also suggested new \emph{supervised} learning algorithms. However, there has been much less work on personalized federated unsupervised learning. We will build on the statistical approach studied in \cite{ozkara2023} for supervised learning, applying it to personalized unsupervised learning. This leads to new federated algorithms for personalized dimensionality reduction and personalized diffusion-based generative models for heterogeneous data.

%A challenge how much to rely on local data and how much on collaborative information; \emph{i.e.,} learning from others, given the heterogeneity.
A question is how to use local data and learn from others despite heterogeneity. The hierarchical Bayes framework \cite{efron_2010} suggests using an estimated population distribution effectively as a prior. However, the challenge is to efficiently estimate it,  enabling each client to combine local data and collaborative information for the unsupervised learning task. We do this by \emph{simultaneously} learning a global model (a proxy for the population model) and learning a local model by adaptively estimating its discrepancy to balance the global and local information. In Section \ref{sec:ProbStat}, we define this through a loss function, making it amenable to a standard distributed gradient descent approach. Our main contributions are as follows.\\
%\noindent\textbf{Main contributions:} {\sf (i)} use the statistical framework to develop loss criteria for personalized unsupervised tasks of dimensionality reduction and novel sample generation {\sf (ii)} develop adaptive algorithms for (non-linear) dimensionality reduction and analyze their convergence behavior {\sf (iii)} develop personalized generation through diffusion models by connecting it to the hierarchical framework and develop adaptive algorithms for personalized generation {\sf (iv)} Numerically evaluate the proposed algorithms using dataset benchmarks. In more detail:\\
\noindent\textbf{Unsupervised collaboration learning criteria:} Section \ref{sec:ProbStat}  develops and uses the hierarchical Bayes framework for personalized dimensionality reduction and personalized  (generative) diffusion models. We develop an \emph{Adaptive Distributed Emperical-Bayes based Personalized Training} (\texttt{ADEPT}) criterion which  embeds the balance between local data and collaboration for these tasks (see below). As far as we are aware, these are the first explicit criteria for such personalized federated unsupervised learning tasks. \\
\noindent\textbf{Personalized dimensionality reduction:} Section \ref{sec:DimRed} develops adaptive personalized algorithms for linear (PCA) (\texttt{ADEPT-PCA}) and non-linear (auto-encoders) (\texttt{ADEPT-AE}) for dimensionality reduction. We also demonstrate its convergence in Theorems \ref{thm:PCA_convergence} and \ref{thm:AE_convergence}. In Remark \ref{remark:ae}, we see that these allow us to theoretically examine the impact of heterogeneity, number of local samples and number of clients, on optimization performance. We believe that these are the first adaptive algorithms for personalized dimensionality reduction and their convergence analyses. Finally in Section \ref{sec:Expts} \footnote{Our code is available at \url{https://github.com/kazkara/adept}.}, we evaluate \texttt{ADEPT-PCA} and  \texttt{ADEPT-AE} on synthetic and real data showing the benefits of adaptive collaboration. For example, \cref{tab:ae_real} shows effective amplification of as much as $20\times$ in local sample complexity through collaboration.\\
\noindent\textbf{Personalized diffusion models:} Section \ref{sec:DiffModel}  develops an adaptive personalized diffusion generative model (\texttt{ADEPT-DGM}) to generate novel samples for local data statistics. We believe that \texttt{ADEPT-DGM} is the first algorithm for federated personalized diffusion. We develop a theory for such personalized federated diffusion models through our statistical framework, and use this to demonstrate, in Theorem \ref{thm:ddpm-improvement}, conditions when collaboration can improve performance, despite statistical heterogeneity. Finally in Section \ref{sec:Expts}, we evaluate \texttt{ADEPT-DGM}, demonstrating the value of adaptive collaboration despite heterogeneity, as well as significant performance benefits for ``worst'' clients.

\noindent\textbf{Related work:}
As mentioned earlier, there have been several recent works on personalized \emph{supervised} learning (see next paragraph). There has been much less attention given to personalized federated unsupervised learning. The closest work to ours on personalized dimensionality reduction is \cite{shi2022personalized,ozkara2023isit} which study personalized PCA algorithms. \cite{shi2022personalized} has a restrictive assumption that principal components for global and local models are non-overlapping. \cite{ozkara2023isit} uses the hierarchical Bayes statistical model to develop a criterion for personalized PCA; however, the authors assume heterogeneity of the setting is known; in contrast, \texttt{ADEPT-PCA} learns and adapts the solution accordingly. There is some literature on specific tasks such as training recommender systems \cite{le2022personalized,Mirko2021}, grouping clients based on latent representations \cite{yang2023personalized}, generating data to improve performance of personalized supervised learning \cite{Cao2023, peng2023}. However, our approach to developing personalized dimensionality reduction for heterogeneous data is distinct from them. We are not aware of any FL work for personalized diffusion generative models. There is work using pre-trained diffusion models and then fine-tuning them \cite{ruiz2023dreambooth, zhang2023adding, moon2022fine, ma2023subject}. However, these do not fit into the federated learning paradigm, and require data collection from clients to obtain the pre-trained model. One way to view our approach is to \emph{simultaneously} build such ``global'' models (akin to pre-trained models) and individual (personalized) models, while not sharing data.

Beyond these above works, some more related works are as follows: Personalized Federated Learning (FL) has seen recent advances with diverse approaches for learning personalized models. These approaches encompass meta-learning-based methods \cite{fallah2020personalized,acar2021debiasing,kohdak2019adaptive}, regularization techniques \cite{deng2020adaptive,mansour2020approaches,hanzely2020federated}, clustered FL \cite{zhang2021personalized, mansour2020approaches, ghosh2020efficient,marfoq2021federated}, knowledge distillation strategies \cite{li2020federated,ozkara2021quped}, multi-task learning \cite{dinh2020personalized,smith2017federated,vanhaesebrouck2017decentralized,zantedeschi2020fully}, and the utilization of common representations \cite{du2021fewshot, tian2020rethinking, collins-icml21}. Additionally, recently there have been works on using a hierarchical Bayesian view to derive novel personalized supervised FL algorithms \cite{ozkara2023, chen2022selfaware, kotelevskii2022fedpop}. Adaptation is also considered in \cite{ozkara2023, chen2022selfaware}.
In \cite{chen2022selfaware}, variance estimation is performed for supervised learning to estimate the heterogeneity within the local models and the estimated variance is used to form an initialization for each local model to be trained on. However, if we apply variance estimation to our criterion, we observe an early stopping issue during the training. In contrast, ours is based on a standard gradient descent. Moreover, they do not examine convergence which we do in our unsupervised algorithms. In \cite{ozkara2023}, the authors consider an adaptation for supervised learning based on the KL-divergence criterion and do not have hyper prior. Moreover, a convergence analysis is not explored in their algorithm. In that case, the $\sigma$ is inclined to smaller values or even vanishes during the training, and our \texttt{ADEPT} criterion resolves the issue by introducing the hyper prior. \cite{tun2023federated} investigated local features of globally trained diffusion models through FedAvg, but they do not consider personalized generation. %We did not find any other works on personalized federated generation models.

\section{Problem Formulation}
\label{sec:ProbStat}
\label{sec:prob-form}
%In this section, we present a hierarchical Bayesian framework for the personalized unsupervised learning. We then cast the problems of personalized dimension reduction (principal component analysis and auto-encoder decoder) and personalized generation by diffusion model into the proposed framework as specific instances.
We present a hierarchical Bayesian framework for the personalized unsupervised learning, using it to develop optimization criteria for federated personalized dimensionality reduction and personalized diffusion models. First we state some preliminaries for notation. For the notations used in this paper:
\begin{itemize}
\item We use bold lowercase letters (such as $\boldsymbol{u}$, $\boldsymbol{v}$) to denote vectors and we use bold uppercase letters (such as $\boldsymbol{U}$, $\boldsymbol{V}$) to denote matrices.

\item Given a composite function $f(\vct{u}, \vct{v})$, we denote $\nabla f(\vct{u}, \vct{v})$ or $\nabla_{(\vct{u}, \vct{v})} f(\vct{u}, \vct{v})$ as the gradient; $\nabla_{\vct{u}} f(\vct{u}, \vct{v})$ and $\nabla_{\vct{v}} f(\vct{u}, \vct{v})$ as the partial gradients with respect to $\vct{u}$ and $\vct{v}$.

\item For a vector $\vct{u}$, $\Vert \vct{u} \Vert$ denotes the $\ell_2$-norm $\Vert \vct{u} \Vert_2$. For a matrix $\boldsymbol{U}$, $\Vert \bU \Vert_2$ and $\Vert \bU \Vert_{op}$ both denote the $\ell_2$-norm (operator norm) of the matrix and $\Vert \bU \Vert, \Vert \bU \Vert_F$ denotes the Frobenius norm of the matrix.

\item We use $\{\bU_i\}_{i=1}^m$ or $\{\bU_i\}$ (when the context is clear) to denote the collection $\{ U_1, \dots, U_m\}$. When there are multiple indices, we use $\{\bU_{i,t}\}_i \coloneqq \{\bU_{1,t}, \dots, \bU_{m,t}\}$ to denote the collection over index $i$.
\end{itemize}

%\subsection{Personalized Unsupervised Learning from a Hierarchical Bayesian View}
\subsection{Hierarchical Bayes for personalized learning}
\begin{figure}[t]
	\centering
	\includegraphics[width=0.4\textwidth]{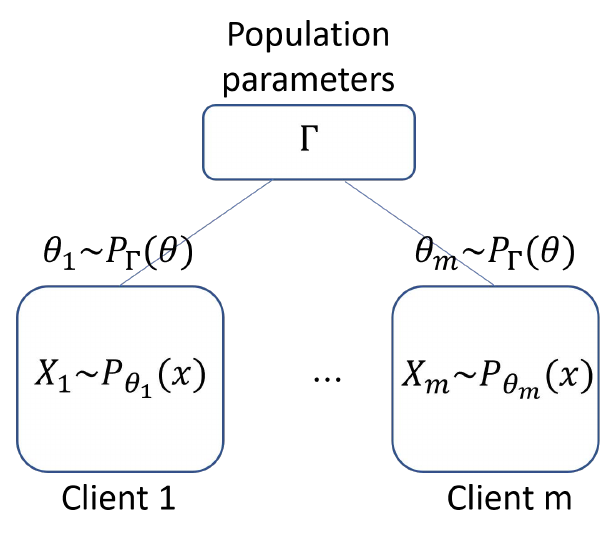}
	\vspace{-12pt}
	\caption{Hierarchical Bayesian model of data distribution \label{fig:model}}
	\vspace{-12pt}
\end{figure}

%The goal of unsupervised learning is to reveal insights on the true underlying generative process of the data, and utilize/generalize those insights to solve tasks such as dimensionality reduction or new data generation. We study the personalized unsupervised learning, where there are $m$ clients each with a (typically small) local dataset of size $n$. The datasets are heterogeneous across clients (sampled from different distributions), and each client is solving a local unsupervised learning task. Solving the local task by only leveraging the local dataset cannot guarantee performance due to the small number of samples, this necessitates collaboration to increase effective samples size of the client. Collaboration among clients is useful only if the (distinct) local distributions are related e.g. through a population model. To this end, we introduce a hierarchical Bayesian model that not only justifies the collaboration but also leads to a systematic approach for personalized unsupervised learning.

The \emph{statistical model} based on hierarchical Bayes, suitable for federated unsupervised learning,  is illustrated in Figure \ref{fig:model}. There are $m$ clients. Each client $i$ is associated with  a parameter $\vct{\theta}_i$ and has a local dataset $\mat{X}_i$ consisting of $n$ data points $(\vct{x}_{i1},\ldots,\vct{x}_{in} )$ i.i.d. from distribution $p(\vct{x} | \vct{\theta}_i)$. The parameters obey a population distribution (a.k.a. prior distribution in the Bayesian model) $p(\vct{\theta} | {\Gamma})$ parameterized by $\Gamma$. We have a (carefully designed) hyper prior distribution $\pi$ over $\Gamma$, i.e., $\Gamma \sim \pi$ to prevent ill-posedness. This statistical model defines the joint distribution 
\begin{align}\label{eq:learning-generative-model}
	p_{\Gamma, \{\vct{\theta}_i\},\{\mat{X}_i\}}(\Gamma, \vct{\theta}_1,\ldots,\vct{\theta}_m,\mat{X}_1,\ldots,\mat{X}_m) 
	= \pi(\Gamma) \prod_{i=1}^mp(\vct{\theta}_i | \Gamma)\prod_{i=1}^m\prod_{j=1}^n p(\vct{x}_{ij} | \vct{\theta}_i).
\end{align}
We learn the distribution parameters $\Gamma, \{\vct{\theta}_i\}$ from data $\{\mat{X}_i\}$ by maximizing the joint distribution (a.k.a. maximum a posteriori), by minimizing the \texttt{ADEPT} loss function:
\begin{align*}
f( \{ \vct{\theta}_i\}, \Gamma; \{\mat{X}_i\} ) = \frac{1}{m}\sum_{i = 1}^m f_i(\vct{\theta}_i; \mat{X}_i) + R(\{\vct{\theta}_i\}, \Gamma),
\end{align*}
where $f_i(\vct{\theta}_i; \mat{X}_i) := - \sum_{j=1}^n \log(p(\vct{x}_{ij} | \vct{\theta}_i))$ is the local loss function for client $i$ reflecting the likelihood of the local dataset, and $R(\{\vct{\theta}_i\}, \Gamma) = - \frac{1}{m} \log \pi(\Gamma) - \frac{1}{m}\sum_{i = 1}^m \log p(\vct{\theta}_i | \Gamma) $ is a regularization allowing the collaboration among clients. When the likelihood function is not easy to optimize, we leverage surrogates such as evidence lower bound (ELBO) instead. 

In this work, we focus on a Gaussian population distribution over an $d_\theta$-dimensional normed metric space $(\Theta, \|\cdot\|)$\footnote{Note that this general framework can be applied to any general (parametric) population distribution.}. Specifically, $\Gamma = (\vct{\mu}, \sigma)$, the parameters $\vct{\theta}, \vct{\mu} \in \Theta$ and $p(\vct{\theta} | \Gamma) = \frac{1}{(2\pi \sigma^2)^{d_\theta/2}} \exp(- \frac{ \| \vct{\theta} - \vct{\mu} \|^2 }{2 \sigma^2})$, where $d$ is the dimensionality of $\vct{\theta}$ and $\vct{\mu}$ has the same dimension as $\vct{\theta}$. We assume an improper (non-informative) hyper prior $\pi$ over $\Gamma =  (\vct{\mu}, \sigma)$, as $\vct{\mu} \sim \mathcal{N}(0, \infty \cdot \bI_{d_\theta})$ and $\sigma^2$ follows an inverse gamma distribution parameterized by a hyper-parameter $\xi$, i.e., $\pi(\vct{\mu}, \sigma) \propto \exp(\frac{m\xi}{\sigma^2})$ \footnote{Inverse-Gamma distribution is conjugate prior distribution for the variance term.}. We thus have the regularization
\begin{align}
R(\{\vct{\theta}_i\}, \Gamma) & = - \frac{1}{m} \log \pi(\Gamma) - \frac{1}{m}\sum_{i = 1}^m \log p(\vct{\theta}_i | \Gamma) %\notag \\ &
 = \frac{1}{m} \sum_{i = 1}^m \frac{2 \xi + \| \vct{\mu} - \vct{\theta}_i\|^2}{2\sigma^2} + d_\theta \log \sigma. \label{eqn:regularizer}
\end{align}

\subsection{Personalized Dimensionality Reduction}
\label{subsec:PersDimRed}
\paragraph{Linear Dimensionality Reduction:}
The linear dimensionality reduction is equivalent to PCA. We extend a personalized PCA formulation that was previously studied in \cite{ozkara2023isit} by introducing adaptivity i.e. optimizing the loss over $\sigma$. In this setting, the dataset from client $i$ is $\mat{X}_i \in \mathbb{R}^{d \times n}$ containing $n$ samples of $d$ dimensional vectors. Let us denote $\bS_i = \frac{1}{n}\mat{X}_i\mat{X}_i^\top$ as the sample covariance matrix of client $i$. For notational consistency with canonical PCA notation, we set the parameters $\vct{\theta}_i = \bU_{i} \in \mathbb{R}^{d \times r}$. Similar to \cite{ozkara2023isit}, we adopt the probabilistic view of PCA \cite{tipping1999probabilistic}, 
\begin{align}\label{eq:latent model}
	\vspace{-0.3cm}
	\vct{x}_{ij} = \bU_{i} \vct{z}_{ij} + \vct{\epsilon}_{ij},
	\vspace{-0.5cm}
\end{align}
%where $\vct{z}_{ij} \in \mathbb{R}^r$, $\vct{\epsilon}_{ij} \in \mathbb{R}^d$, $Z_{i1}, \ldots,Z_{in} \stackrel{i.i.d.}{\sim} \mathcal{N}(\boldsymbol{0},\bI_r)$ and $N_{i1}, \ldots, N_{in} \stackrel{i.i.d.}{\sim} \mathcal{N}(\boldsymbol{0}, \sigma_\epsilon^2 \bI_d)$. 
where $\vct{z}_{i1}, \ldots,\vct{z}_{in} \stackrel{i.i.d.}{\sim} \mathcal{N}(\boldsymbol{0},\bI_r)$ and $\vct{\epsilon}_{i1}, \ldots, \vct{\epsilon}_{in} \stackrel{i.i.d.}{\sim} \mathcal{N}(\boldsymbol{0}, \sigma_\epsilon^2 \bI_d)$. 
This results in the likelihood function $p(\mat{X}_i | \vct{\theta}=\bU_{i}) =  \mathcal{N}(\boldsymbol{0},\bU_{i}\bU_{i}^\top+\sigma^2_{\epsilon}\bI)$. Recall that the prior parameter is $\Gamma = (\vct{\mu}, \sigma)$, here for notational consistency we use $\bV = \vct{\mu}$. The underlying metric space $(\Theta, \|\cdot\|)$ containing the parameters $\bU$ and $\bV$ in the Steifel manifold where $St(d,r) \coloneqq \{\boldsymbol{U} \in \mathbb{R}^{d \times r}|\boldsymbol{U}^{\top}\boldsymbol{U}=\boldsymbol{I}\}$. The metric is $d(\bV, \bU) = \left\Vert P_{\mathcal{T}_{\bV}}(\bU) \right\Vert$ where $P_{\mathcal{T}_{\bV}}(\bU) = \bU-\frac{1}{2}\bV(\bV^{\top}\bU+\bU^{\top}\bV)$ is the projection of $\bU$ onto the tangent space at $\bV$. Because computing the (geodesic) distance on $St(d,r)$ is hard, the defined metric first projects a matrix to the tangent space of another matrix and then computes the Frobenius norm. The personalized unsupervised learning is then minimizing the \texttt{ADEPT-PCA} loss function:
\begin{align}\label{eq:pca_opt_problem}
	\min_{\{\bU_i\}, \bV, \sigma} f^{\text{pca}}(\{\bU_i\} , \bV, \sigma ; \{\mat{X}_i\} ):= \frac{1}{m}\Big(\sum_{i=1}^m \frac{n}{2}(\log(|\bW_i|) 
	 +\text{tr}(\bW_i^{-1}\bS_i)) + \frac{2 \xi + d^2(\bV,\bU_i)}{2\sigma^2} \Big) + d_\theta \log \sigma
\end{align}
s.t. $\bV^{\top}\bV = \boldsymbol{I}, \; \bU_i^{\top}\bU_i = \boldsymbol{I}  \quad \forall i \in [m]$; where $\bW_i = (\bU_i\bU_i^\top+\sigma^2_{\epsilon}\boldsymbol{I})$.

\paragraph{Non-linear dimensionality reduction:} While PCA is a good starting point for dimensionality reduction, it cannot capture non-linear relations between the latent variable and the observed space. Hence, one can extend \eqref{eq:latent model} to model non-linearity as follows,
\begin{align}\label{eq:nonlin latent model}
	\vspace{-0.3cm}
	X_{ij} = \psi_{\vct{\theta}_i^{\mathrm{d}}}( \vct{z}_{ij} ) + \vct{\epsilon}_{ij},
	\vspace{-0.5cm}
\end{align} 
where $\vct{\theta}_i^{\mathrm{d}}$ parameterizes a non-linear decoding map from the latent space to the observed space. We can parameterize the encoder structure such that $ \vct{z}_{ij} = g_{\vct{\theta}_i^{\mathrm{e}}}(X_{ij})$. Using Gaussian distribution we get the \texttt{ADEPT-AE} criterion:
\begin{align}\label{eq:opt_problem_ae}
	\argmin_{\{\vct{\theta}_i\}, \vct{\mu} , \sigma} f^{\text{ae}}(\{\vct{\theta}_i\}, \vct{\mu} , \sigma)   = \frac{1}{m}\sum_{i=1}^m \Big( \|\mat{X}_i - \psi_{\vct{\theta}_i^{\mathrm{d}}}( g_{\vct{\theta}_i^{\mathrm{e}}} (\mat{X}_i) ) \|^2_F   + \frac{2 \xi + \|\vct{\mu} - \vct{\theta}_i\|^2}{2\sigma^2} \Big)
  	+ d_\theta \log \sigma
\end{align}
where $\vct{\mu}$ is the global model and $\{\vct{\theta}_i\}$ are the autoencoders of the clients and viewed as the concatenation of $\vct{\theta}_i^{\mathrm{e}}$ and $\vct{\theta}_i^{\mathrm{d}}$.

%\paragraph{Non-linear dimensionality reduction} While PCA is a good starting point for dimensionality reduction, it cannot capture non-linear relations between the latent variable and the observed space. Hence, one can extend \eqref{eq:latent model} to model non-linearity as follows,
%\begin{align}\label{eq:nonlin latent model}
%	\vspace{-0.3cm}
%	\boldsymbol{x} = \psi_{\theta_i^{d}}(\boldsymbol{z})+\boldsymbol{\epsilon}
%	\vspace{-0.5cm}
%\end{align} 
%where $\theta_i^{d}$ parameterizes a non-linear decoding map from the latent space to observed space. Furthermore, by parameterizing the encoder structure such that $\bz = h_{\theta_i^{e}}(\boldsymbol{x})$ we obtain,
%\begin{align}\label{eq:nonlin latent model2}
%	\vspace{-0.3cm}
%	\boldsymbol{x} = g_{\theta_i}(\boldsymbol{x})+\boldsymbol{\epsilon}
%	\vspace{-0.5cm}
%\end{align} 
%Gaussian assumptions on data generating process, as well as the clients' true models yields the following MLE problem:
%\begin{align}\label{eq:opt_problem_ae}
%	\argmin_{\{\theta_i\}, \mu , \sigma}  &\frac{1}{m}\sum_{i=1}^m f^{(ae)}_i(\mu, \{\theta_i\}, \sigma)\\
%	&:=  \frac{1}{m}\sum_{i=1}^m \|\mat{X}_i - g_{\theta_i}(\mat{X}_i)\|^2_F 
%  	+\frac{2 \xi + \|\mu - \theta_i\|^2}{2\sigma^2}\\
%  	& \quad + d \log \sigma \notag 
%\end{align}
%where $\mu$ is represents the global model $\{\theta_i\}$ are the autoencoders of the clients.

\subsection{Personalized Generation through Diffusion Models}
\label{sec:prob-form-ddpm}
%\paragraph{The Denoising Diffusion Model Basics:}
The denoising diffusion model has attracted attention recently due to its capability of generating high-quality images and the theoretical foundation of the stochastic differential equations \cite{rezende2015variational,sohl2015}. The mathematical model of a diffusion model is the following stochastic differential equation \cite{song2020score}
\begin{align}
	d \vct{x}_t = d \vct{w}_t \quad \text{(variance exploding process)}, \label{eqn:diffusion-ve}
\end{align}
where $\vct{w}_t$ is a standard $d$-dimensional Brownian motion. For time $t \in [0, T]$, its time-reversed process is
\begin{align}
	d \vct{x}_t^{\leftarrow} = \nabla \ln p_{\vct{x}_{T-t}}(\vct{x}_t^{\leftarrow}) dt + d \vct{w}^{\leftarrow}_t, \label{eqn:reverse-time}
\end{align}
where $p_{\vct{x}_t}$ is the probability density function of $\vct{x}_t$, $\vct{w}^{\leftarrow}_t$ is a standard Wiener process. It can be verified that $X^{\leftarrow}_t$ follows the same distribution as $\vct{x}_{T- t}$. More strongly, suppose $X^{\leftarrow}_0$ has the same distribution as $\vct{x}_{T}$, and then the two processes $\{\vct{x}_{T - t}\}_{t \in [0, T]}$ and $\{\vct{x}^{\leftarrow}_t\}_{t \in [0, T]}$ have the same distribution. 

Suppose that the score function $\nabla \ln p_{\vct{x}_{t}}( \vct{x} )$ can be represented by a neural network $\vct{\phi}(\vct{x} ; \vct{\theta}, t) \in \mathbb{R}^{d}$ for some parameter $\vct{\theta}$. The data generation is then integration over the time-revised process \eqref{eqn:reverse-time} with the drift function $\vct{\phi}(\vct{x}; \vct{\theta}, t)$ and the starting distribution $\vct{x}_0^{\leftarrow} \sim \mathcal{N}(\vct{0}, \sigma_0^2 \bI_d)$. The generated distribution $p(\vct{x} | \vct{\theta})$ is then implicitly defined without a closed form. \cite[Eq (4)]{kingma2023understanding} shows that 
\begin{small}
\begin{align}
 \ln p(\vct{x} | \vct{\theta}) \geq ELBO_{\vct{\theta}}(\vct{x}) \label{eqn:ELBO-def} = -\frac{1}{2}\int_{t = 0}^T \mathbb{E} \| \vct{\phi}(\vct{x}_t; \vct{\theta}, t) - \nabla_{\vct{x}_t} \ln p_{\vct{x}_t | \vct{x}_0}(\vct{x}_t | \vct{x}) \|^2 d t + C 
	 %&=  - \frac{1}{2}\int_{t = 0}^T \mathbb{E} \left[ \| \vct{\phi}(\vct{x} + \sqrt{t} \vct{\epsilon}; \vct{\theta}_j, t) + \vct{\epsilon} / \sqrt{t}\|^2 \right] d t + C, 
\end{align}
\end{small}
where $\vct{x}_t \sim \mathcal{N}(\vct{x}, \bI_d)$ %, $\vct{\epsilon} \sim \mathcal{N}(\vct{0}, \bI_d)$ 
and $C$ is some constant factor independent of $\vct{\theta}$. Accordingly, we use ELBO \footnote{Unlike the discrete case, for the continuous case, there are several definitions used for ELBO e.g. in \cite{song2021denoising,kingma2021on}; \cite{kingma2023understanding} uses the simplified definition in \eqref{eqn:ELBO-def} to unify different ELBO definitions.} as a surrogate function for the negative log-likelihood; thus, the personalized loss function is $f_i(\vct{\theta}_i ; \mat{X}_{i} ) = -ELBO_{\vct{\theta}_i}(\mat{X}_{i}) = - \sum_{j = 1}^n ELBO_{\vct{\theta}_i}(\vct{x}_{ij})$. 
The personalized data generation is then minimizing the \texttt{ADEPT-DGM} loss function
\begin{small}
\begin{align}
	 \min_{\{\vct{\theta}_{i} \}, \vct{\mu}, \sigma^2} f^{\text{df}}( \{\vct{\theta}_{i} \}, \vct{\mu}, \sigma; \{\mat{X}_i\} ) := 
	 \frac{1}{m} \sum_{i = 1}^m \left( - ELBO_{\vct{\theta}_i}(\mat{X}_{i})  + \frac{2 \xi + \| \vct{\mu} - \vct{\theta}_i\|^2}{2\sigma^2} \right) + d_\theta \log \sigma. \label{eqn:loss-df} 
\end{align}
\end{small}

\vspace{-10pt}
\section{Personalized Federated Dimensionality Reduction}
\label{sec:DimRed}
In this section, we discuss our algorithms and convergence results on personalized adaptive dimensionality reduction.
\subsection{Personalized Adaptive PCA: \texttt{ADEPT-PCA}}

\begin{algorithm}[htb]
	%\caption{Adaptive Alternating Steifel Gradient Descent on the Steifel manifold for optimizing \eqref{eq:pca_opt_problem}}
	\caption{\texttt{ADEPT-PCA} Algorithm}
	{\bf Input:} Number of iterations $T$ and learning rates ($\eta_1, \eta_2, \eta_3$).
	\begin{algorithmic}[1] 	\label{algo:pca}
		\STATE \textbf{Initialize} local PCs $\{\bU_{i,0}\}_{i=1}^m$, global PC $\bV_{0}$, and $\sigma_0$.
		\STATE Broadcast $\bV_{0}, \sigma_{0}$ to the clients
		\FOR{$t=1$ \textbf{to} $T$} 
		\STATE \textbf{On the clients:}
		\FOR {$i=1$ \textbf{to} $m$:}
		\STATE Receive $\bV_{t-1}, \sigma_{t-1}$
		\STATE $\sigma_{i,t} = \sigma_{t-1} - \eta_3 \pt{\sigma_{t-1}} f_i^{\mathrm{pca}}(\bU_{i,t-1}, \bV_{t-1}, \sigma_{t-1})$
		\STATE $\bg^{\bU}_{i,t} = P_{\mathcal{T}_{\bU_{i,t-1}}}(\nabla_{\bU_{i,t-1}} f_i^{\mathrm{pca}}(\bU_{i,t-1}, \bV_{t-1}, \sigma_{t-1}))$
		\STATE $\bU_{i,t}\gets \calR_{\bU_{i,t-1}}(- \eta_1 \bg^{\bU}_{i,t}) $
		\STATE $\bg^{\bV}_{i,t} = P_{\mathcal{T}_{\bV_{t-1}}}(\nabla_{\bV_{t-1}} f_i^{\mathrm{pca}}(\bU_{i,t}, \bV_{t-1}, \sigma_{t-1}))$
		\STATE $\bV_{i,t}\gets \bV_{t-1} - \eta_2 \bg^{\bV}_{i,t}$
		\STATE Send $\bV_{i,t}$, $\sigma_{i,t}$ to the server
		\ENDFOR
		\STATE \textbf{At the server:}
		\STATE Receive $\{\bV_{i,t}\}_{i=1}^m$ and $\{ \sigma_{i,t} \}_{i=1}^m$
		\STATE $\bV_{t} = \calR_{\bV_{t-1}} \left( \frac{1}{m} \sum_{i=1}^m\bV_{i,t}-\bV_{t-1} \right)$
		\STATE $\sigma_t = \frac{1}{m} \sum_{i=1}^m \sigma_{i,t}$
		\STATE Broadcast $\bV_{t}, \sigma_t$ to the clients
		\ENDFOR
	\end{algorithmic}
	{\bf Output:} Personalized PCs $\{ \bU_{1,T},\ldots,\bU_{m,T} \}$.
\end{algorithm}
We introduce \texttt{ADEPT-PCA} (\cref{algo:pca}) to train adaptive personalized PCA. On \textbf{line 7} we update the adaptivity parameter $\sigma_{i,t}$, on \textbf{lines 8, 10} we compute the projected gradients for personalized and global parameters. On \textbf{lines 9, 16} we update the personalized and global parameters through projected gradient descent and retraction. In~\cref{algo:pca}, we use the polar retraction to map the updated $\bV$ and $\bU_i$ back to the Steifel manifold. We define the polar retraction in the next section in detail.

To show the convergence of the algorithm we need the following standard assumption and naturally occurring lower bound on $\sigma$.

\begin{assumption}\label{assumption}
	For each client $i$, the operator and Frobenius norms of $\bS_i$ are bounded by
	\begin{align*}
		\|\bS_i\|_F \leq G_{i,F} \quad \text{and} \quad \|\bS_i\|_{op} \leq G_{i,op},
	\end{align*}
	and $G_{max,F} \coloneqq \max_{i \in [m]} G_{i,F}$,$G_{max,op} \coloneqq \max_{i \in [m]} G_{i,op}$. The assumption implies the Lipschitz smoothness properties of the loss function w.r.t. personalized PCs $\{\bU_i\}$.
\end{assumption}
\begin{lemma}[A lower bound on $\sigma_t$] \label{lem:sigma_lb}
	Given any $\omega \in (0,1)$. Let the learning rate $\eta_3 \leq (1-\omega) \frac{2 \xi}{d_\theta^2}$ and the initialization $\sigma_0 \geq \omega \sqrt{\frac{2 \xi}{d_\theta}}$. Then, for all $t \in [T]$, we have $\sigma_t \geq \omega \sqrt{\frac{2 \xi}{d_\theta}}$.
\end{lemma}
See Appendix~\ref{app:lem sigma lb} for the proof. We will fix some $\omega \in (0,1)$ for the rest of the paper and initialize $\sigma_0$ accordingly so that we can utilize the lower bound in Lemma~\ref{lem:sigma_lb}. The bound is due to $\xi$ to guarantee that the loss does not explode due to vanishing $\sigma$. Let us now define $\bg^{U}_{i,t} = P_{\mathcal{T}_{\bU_{i,t-1}}} ( \nabla_{\bU_{i,t-1}} f_i^{\text{pca}}(\bU_{i,t-1},\bV_{t-1}, \sigma_{t-1}) )$, $\bg^{V}_{t} = \frac{1}{m} \sum_{i=1}^m P_{\mathcal{T}_{\bV_{t-1}}} ( \nabla_{\bV_{t-1}} f_i^{\text{pca}}(\bU_{i,t},\bV_{t-1}, \sigma_{t-1}) )$, and $g^{\sigma}_{t} = \pt{\sigma_{t-1}} f^{\text{pca}}(\{\bU_{i,t-1}\}_i, \bV_{t-1}, \sigma_{t-1})$. Then, \cref{algo:pca} has the following convergence upper bound for finding a first-order stationary point.
\begin{theorem}[Convergence of \texttt{ADEPT-PCA} \cref{algo:pca}] \label{thm:PCA_convergence}
	Let $G_{t} = \left( \frac{1}{m} \sum_{i=1}^m \big\lVert \bg^{U}_{i,t} \big\rVert^2 \right) + \big\lVert \bg^{V}_{t} \big\rVert^2 + (g^{\sigma}_{t})^2$. By choosing $\eta_1 = \min\{ \frac{1}{3C_{\eta_1}}, 1\}$, $\eta_2 = \min\{ \frac{1}{3C_{\eta_2}}, 1\}$, and $\eta_3 = \min \Big\{ \frac{\eta_1}{3 (L_U^{(\sigma)})^2}, \frac{\eta_2}{3 (L_V^{(\sigma)})^2}, \frac{1}{L_\sigma} \Big\}$, we have
	\begin{equation*}
		\frac{1}{T} \sum_{t=1}^T G_{t}
		\leq \frac{3 \Delta^{\mathrm{pca}}_T}{ T \min \{ \eta_1, \eta_2, \eta_3 \} },
	\end{equation*}
	where $T$ is number of total iterations, $\Delta^{\mathrm{pca}}_T = f^{\mathrm{pca}}(\{\bU_{i,0}\}_i, \bV_{0}, \sigma_{0}) - f^{\mathrm{pca}}(\{\bU_{i,T}\}_i, \bV_T, \sigma_T)$, and $\eta_1,\eta_2,\eta_3$ are the learning rates for updating $\bU_i$, $\bV$, and $\sigma$ respectively. The constants are defined such that $f^{\text{pca}}(\cdot)$ is$L_\sigma$-smooth w.r.t. $\sigma$ and $\bg^{U}_{i,t}, \bg^{V}_{t}$ are $L_{U}^{(\sigma)}, L_{V}^{(\sigma)}$ continuous w.r.t. $\sigma$ respectively. $C_{\eta_1},C_{\eta_2}$ are defined in detail in the next section and depend on smoothness w.r.t. $\bU,\bV$.
\end{theorem}
We provide a detailed comment on the factors impacting the convergence rate in Remark~\ref{remark:ae}.

\subsubsection{Proof Outline for Theorem~\ref{thm:PCA_convergence}}

For any point  $\boldsymbol{U} \in St(d,r)$, a retraction at a point $\boldsymbol{U} \in St(d,r)$ is a map $\mathcal{R}_{\boldsymbol{U}}:\mathcal{T}_{\boldsymbol{U}}\rightarrow St(d,r)$ that induces local coordinates on the Stiefel manifold. In this work, we use polar retraction that is defined as $\mathcal{R}_{\boldsymbol{U}}(\bV)=(\bU+\bV)(\bI+\bV^\top\bV)^{-\frac{1}{2}}$. The polar retraction is a second-order retraction that approximates the exponential mapping up to second-order terms. Consequently, it possesses the following non-expansiveness property and we state the Lemmas and properties that we use throughout the proof, detailed proof of Lemmas can be found in \cref{app:pca}. 

\begin{lemma}[Non-expansiveness of polar retraction \cite{chen2021decentralized}]\label{lem:retraction}
	Let $\bV \in St(d,r)$, for any point $\bU\in \mathcal{T}_{\bV}$ with bounded norm, $\|\bU\|_F\leq M$, there exists $C\in \mathbb{R}$ such that
	\begin{align}
		\|\calR_{\bV}(\bU)-(\bV+\bU)\|_F \leq C \|\bU\|_F^2.
	\end{align}
\end{lemma}

\begin{lemma}[Lipschitz type inequality \cite{chen2021decentralized}]\label{lem:lips}
	Let $\bU, \bV \in St(d,r)$. If a function $\psi$ is $L$-Lipschitz smooth in $\mathbb{R}^{d\times r}$, the following inequality holds:
	\begin{align*}\lvert \psi(\bV) {-} \left( \psi(\bU){+}\langle P_{\mathcal{T}_{\bU}} (\nabla \psi(\bU)), \bV {-} \bU \rangle \right) \rvert \leq \frac{L_g}{2} \|\bV {-} \bU \|_F^2 
	\end{align*}
	where $L_g = L+G$ with $G \coloneqq \max_{\bU\in St(d,r)}\|\nabla \psi(\bU)\|_2$.
\end{lemma} 

% \begin{assumption}\label{assumption}
% 	For each client $i$, the operator and Frobenius norms of $\bS_i$ are bounded by
% 	\begin{align*}
% 		\|\bS_i\|_F \leq G_{i,F} \quad \text{and} \quad \|\bS_i\|_{op} \leq G_{i,op},
% 	\end{align*}
% 	and we define $G_{max,F} \coloneqq \max_{i \in [m]} G_{i,F}$ and $G_{max,op} \coloneqq \max_{i \in [m]} G_{i,op}$.
% \end{assumption}

%\begin{lemma}[A lower bound on $\sigma_t$] \label{lem:sigma_lb}
%	Given any $\omega \in (0,1)$. Let the learning rate $\eta_3 \leq (1-\omega) \frac{2 \xi}{d^2}$ and the initialization $\sigma_0 \geq \omega \sqrt{\frac{2 \xi}{d}}$ in algorithm~\ref{algo:riem-gd}. Then, for all $t \in [T]$, we have
%	\begin{equation*}
%		\sigma_t \geq \omega \sqrt{\frac{2 \xi}{d}}.
%	\end{equation*}
%\end{lemma}

Using \cref{assumption} and \cref{lem:sigma_lb} we introduce the following Lemmas to be used in the proof.

\begin{lemma}[Lipschitz smoothness and bounded gradients with respect to $\sigma$] \label{lem:sigma}
	For all $i \in [m]$ and $\bU_i, \bV \in St(d,r)$, within the domain $\sigma \in [ \omega \sqrt{\frac{2 \xi}{d}}, \infty ] $, the function $f_i^{\mathrm{pca}}(\bU_i, \bV, \sigma)$ is $L_\sigma$-Lipschitz smooth with respect to $\sigma$
	% and $\left\vert \pt{\sigma} f_i^{\mathrm{pca}}(\bU_i, \bV, \sigma) \right\vert \leq G_\sigma$
	with constants
	\begin{gather*}
		L_\sigma \coloneqq \frac{d^2}{2 \xi \omega^2} + \frac{3 d^2}{2 \xi \omega^4} + \frac{3 d^2}{\xi^2 \omega^4}.
		%G_\sigma \coloneqq \frac{d^\frac{3}{2}}{\omega \sqrt{2 \xi}} + \frac{d^\frac{3}{2}}{\omega^3 \sqrt{2 \xi}} + \frac{d^\frac{3}{2} D^2}{\omega^3 (2 \xi)^\frac{3}{2}}.
	\end{gather*}
\end{lemma}

\begin{lemma}[Lipschitz smoothness and bounded gradients with respect to $\bU_i$] \label{lem:U}
	The function $f_i^{\mathrm{pca}}(\bU_i, \bV, \sigma)$ is $L_U$-Lipschitz smooth with respect to $\bU_i$ and $\|\nabla f_i^{\mathrm{pca}}(\bU_i, \bV, \sigma)\|_2 \leq G_U$ for all $i \in [m]$ with constants
	\begin{gather*}
		L_U \coloneqq \frac{n}{2}\left(\frac{1}{\sigma_{\epsilon}^2}+\frac{G_{max,op}}{\sigma_{\epsilon}^4} +\left(1 + \frac{2G_{max,op}}{\sigma_{\epsilon}^2}\right) \frac{2}{\sigma_{\epsilon}^4}\right)+\frac{d}{\xi \omega^2}, \\
		G_U \coloneqq \frac{n}{2}\left(\frac{G_{max,op}}{\sigma_{\epsilon}^4}+\frac{1}{\sigma_{\epsilon}^2}\right) + \frac{d}{\xi \omega^2}.
	\end{gather*}
\end{lemma}

\begin{lemma}[Lipschitz smoothness and bounded gradients with respect to $\bV$]  \label{lem:V}
	The function $f^{\mathrm{pca}}(\{\bU_i\}_i, \bV, \sigma)$ is $L_V$-Lipschitz smooth with respect to $\bV$ and $\|\nabla f^{\mathrm{pca}}(\{\bU_i\}_i, \bV, \sigma)\|_2 \leq G_V$ with constants
	\begin{gather*}
		L_V \coloneqq \frac{12 d}{\xi \omega^2}, \\
		G_V \coloneqq \frac{3 d}{\xi \omega^2}.
	\end{gather*}
\end{lemma}

\begin{lemma}[Lipschitz continuity of $\pt{\sigma} f_i^{\mathrm{pca}}(\bU, \bV, \sigma)$ with respect to $\bU, \bV$]
	\label{lem:UV_sigma_Lip}
	The function $\pt{\sigma} f_i^{\mathrm{pca}}(\bU, \bV, \sigma)$ is $L_U^{(\sigma)}$-Lipschitz continuous with respect to $\bU$ and $L_V^{(\sigma)}$-Lipschitz continuous with respect to $\bV$ with
	\begin{gather*}
		L_U^{(\sigma)} = \frac{\sqrt{2 d^3}}{\omega^3 \sqrt{\xi^3}}, \nonumber \\
		L_V^{(\sigma)} = \frac{2 \sqrt{d^3}}{\omega^3 \sqrt{2 \xi^3}}.
	\end{gather*}
\end{lemma}

Before showing the convergence results, we define the following terms. Let
\begin{gather*}
	C_{\eta_1} = C_1G_1+\frac{ (L_U + G_U) ( C_1^2G_1^2 + 1) }{2}, \\
	C{\eta_2} = C_2G_2+\frac{ (L_V + G_V)  ( C_2^2G_2^2 + 1) }{2}, \\
	G_1 = 2 G_U \sqrt{d}, \\
	G_2 = 2 G_V \sqrt{d}. \numberthis \label{eq:pca_constants}
\end{gather*}
with some constants $C_1, C_2$ given by Lemma~\ref{lem:retraction} and $G_U, G_V$ given in Lemma~\ref{lem:U},~\ref{lem:V}. Given the above results, we can obtain overall sufficient decrease as follows,

\begin{lemma}[Sufficient Decrease]\label{lem:suff_dec}
	At any iteration $t$, we have
	\begin{align*}
		&f^{\mathrm{pca}}(\{\bU_{i,t}\}_i, \bV_t, \sigma_t)-f^{\mathrm{pca}}(\{\bU_{i,t-1}\}_i, \bV_{t-1}, \sigma_{t-1})  \\
		\leq&\; \left( -\eta_1 + C_{\eta_1} \eta_1^2 + \eta_3 (L_U^{(\sigma)})^2 \right) \frac{1}{m} \sum_{i=1}^m \lVert P_{\mathcal{T}_{\bU_{i,t-1}}} \left( \nabla_{\bU_{i,t-1}} f_i^{\mathrm{pca}}(\bU_{i,t-1}, \bV_{t-1},, \sigma_{t-1}) \right) \rVert_F^2 \\
		& + \left( -\eta_2 + C{\eta_2} \eta_2^2 + \eta_3 (L_V^{(\sigma)})^2 \right) \lVert P_{\mathcal{T}_{\bV_{t-1}}} \left( \nabla_{\bV_{t-1}} f^{\mathrm{pca}}(\{\bU_{i,t}\}_i, \bV_{t-1}, \sigma_{t-1}) \right) \rVert_F^2 \\
		& + \left( \frac{- \eta_3 + \eta_3^2 L_\sigma}{2} \right) \left[ \pt{\sigma_{t-1}} f^{\mathrm{pca}}(\{\bU_{i,t-1}\}_i, \bV_{t-1}, \sigma_{t-1}) \right]^2.
	\end{align*}
\end{lemma}

\textit{Proof outline of Theorem~\ref{thm:PCA_convergence}.} We show that a lower bound on $\sigma_t$ is obtained with a proper initialization of $\sigma_0$ and we can further derive the Lipschitz constants of the loss function~\eqref{eq:pca_opt_problem}. The sufficient decrease of the loss function w.r.t. $\bU_i$, $\bV$, and $\sigma$ (Lemma~\ref{lem:suff_dec}) are derived individually using non-expansiveness of polar retraction (\cref{lem:retraction}) and Lipschitz type inequality (\cref{lem:lips}). After the sufficient decrease, we show that by choosing $\eta_1 = \min\{ \frac{1}{3C_{\eta_1}}, 1\}$, $\eta_2 = \min\{ \frac{1}{3C{\eta_2}}, 1\}$, and $\eta_3 = \min \Big\{ \frac{\eta_1}{3 (L_U^{(\sigma)})^2}, \frac{\eta_2}{3 (L_V^{(\sigma)})^2}, \frac{1}{b L_\sigma} \Big\}$, we have. The proofs of Lemmas can be found in Appendix~\ref{app:pca}. Technical challenges in this proof are to control the error due to projections, utilize the lower bound on $\sigma$ to avoid non-smoothness, and use the Lipschitz continuity of the gradients to combine updates on PCs with the update on $\sigma$.

\subsection{Personalized Adaptive AEs: \texttt{ADEPT-AE}}

\begin{algorithm}[h!]
	\caption{\texttt{ADEPT-AE} Algorithm }
	{\bf Input:} Number of iterations $T$, learning rates ($\eta_1, \eta_2, \eta_3$), number of local iterations $\tau$\\
	\vspace{-0.3cm}
	\begin{algorithmic}[1]	 \label{algo:ae-gd}
		\STATE \textbf{Init} local models $\{\vct{\theta}_{i,0}\}_{i=1}^m$, global model $\vct{\mu}_0$, and $\sigma_0$.
		\STATE \textbf{On server:}
		\STATE Broadcast $\vct{\mu}_{0}$, $\sigma_{0}$ to all clients
		\FOR{$t=1$ \textbf{to} $T$}
		\STATE \textbf{On Clients:}
		\FOR {$i=1$ \textbf{to} $m$}
		\IF{$\tau$ divides $t-1$}
		\STATE Receive $\vct{\mu}_{t-1}$, $\sigma_{t-1}$
		%		\STATE $\sigma_{i,t} = \sigma_{t-1} - \eta_{2} \pt{\sigma_{t-1}} f_i^{\mathrm{ae}}(\vct{\theta}_{i, t-1}, \vct{\mu}_{t-1}, \sigma_{t-1})$
		\ENDIF
		\STATE $\vct{\theta}_{i,t} {=} \vct{\theta}_{i,t-1} {-} \eta_{1} \nabla_{\vct{\theta}_{i,t-1}} f_i^{\mathrm{ae}}(\vct{\theta}_{i,t-1}, \vct{\mu}_{t-1}, \sigma_{t-1})$
		\IF{$\tau$ divides $t$}
		\STATE $\vct{\mu}_{i,t} = \vct{\mu}_{t-1} - \eta_{2} \nabla_{\vct{\mu}_{t-1}} f_i^{\mathrm{ae}}(\vct{\theta}_{i,t}, \vct{\mu}_{t-1}, \sigma_{t-1})$
		\STATE $\sigma_{i,t} = \sigma_{t-1} - \eta_{3} \nabla_{\sigma_{t-1}} f_i^{\mathrm{ae}}(\vct{\theta}_{i, t}, \vct{\mu}_{t-1}, \sigma_{t-1})$
		\STATE Send $\vct{\mu}_{i,t}, \sigma_{i,t}$ to server
		\ELSE
		\STATE $\vct{\mu}_{t} = \vct{\mu}_{t-1}$, $\sigma_{t} = \sigma_{t-1}$
		\ENDIF
		\ENDFOR
		\STATE \textbf{At the Server:}
		\IF{$\tau$ divides $t$}
		\STATE Receive $\{ \vct{\mu}_{i,t} \}_{i=1}^m$ and $\{ \sigma_{i,t} \}_{i=1}^m$
		\STATE $\vct{\mu}_{t} = \frac{1}{m} \sum_{i=1}^m \vct{\mu}_{i,t}$, $\sigma_{t} = \frac{1}{m} \sum_{i=1}^m \sigma_{i,t}$
		\STATE Broadcast $\vct{\mu}_{t}$, $\sigma_{t}$ to all clients
		\ENDIF
		\ENDFOR
	\end{algorithmic}
	{\bf Output:} Personalized autoencoders $\{ \vct{\theta}_{1,T}, \ldots, \vct{\theta}_{m,T} \}$.
\end{algorithm}
Algorithm~\ref{algo:ae-gd} shows the alternating gradient descent training procedure for personalized adaptive AEs. At the beginning of local iterations (\textbf{line 8}) the clients receive the global model and $\sigma$\footnote{In the experiments we use individual variance terms for each weight that is $\vct{\sigma}\in \Theta$ which only has a constant effect on convergence.} terms then do local updates on $\vct{\theta}_i$ (\textbf{line 10}). At the end of local iterations, the client updates the global model and variance term using its personalized model \textbf{lines 12,13} and sends them to the server where it is aggregated and broadcast again (\textbf{lines 21-23}).
\begin{assumption} \label{assum:ae_theta_lip}
	The loss function $f^{\mathrm(ae)}_i(\{\vct{\theta}_i\}, \vct{\mu}, \sigma)$ is $L_{\vct{\theta}}$-smooth w.r.t. individual $\{\vct{\theta}_i\}$, $L_{\vct{\mu}}$-smooth w.r.t. $\vct{\mu}$ and $L_\sigma$-smooth w.r.t. $\sigma$. Note that only the first one is an assumption, second and third ones are derived from the fact that $\sigma$ is lower bounded when initialized properly (Appendix~\ref{app:ae}).
\end{assumption}
\vspace{-6pt}
Define	$\bg^{\vct{\theta}}_{i,t} = \nabla_{\vct{\theta}_{i,t-1}} f_i^{\mathrm{ae}}(\vct{\theta}_{i,t-1}, \vct{\mu}_{t-1}, \sigma_{t-1})$, 	$\bg^{\vct{\mu}}_{t} = \nabla_{\vct{\mu}_{t-1}} f^{\mathrm{ae}}(\{\vct{\theta}_{i,t}\}_i$,$ \vct{\mu}_{t-1}, \sigma_{t-1}), 	g^{\sigma}_{t} = \frac{\partial f(\{\vct{\theta}_{i, t}\}_i, \vct{\mu}_{t-1}, \sigma_{t-1})}{\partial \sigma_{t-1}} $.  For \cref{algo:ae-gd}, we obtain the following convergence upper bound for finding a first-order stationary point.
\begin{theorem} [Convergence of \texttt{ADEPT-AE} (\cref{algo:ae-gd})] \label{thm:AE_convergence}
	Let us define $G_{t} 
	= \left( \frac{1}{m} \sum_{i=1}^m \|\bg_{i,t}^{\vct{\theta}}\|^2 \right)
	+ \|\bg_{t}^{\vct{\mu}}\|^2
	+ (g_{t}^{\sigma})^2$. Then, by choosing $\eta_1 = \frac{1}{L_{\vct{\theta}}}$, $\eta_2 = \frac{1}{L_\sigma + {L_\sigma^{(\vct{\mu})}}^2}$, and $\eta_3 = \min \{ 1, \frac{1}{L_{\vct{\mu}}} \}$, we have
	\begin{align*}
		\min_{ \substack{t\in [T] ,\tau\mid t}} \{ G_t \} 
		\leq \frac{\max\{L_{\vct{\theta}},L_{\sigma}+{L_{\sigma}^{(\vct{\mu})}}^2,L_{\vct{\mu}},1 \}\Delta^{\mathrm{ae}}_T}{R},
	\end{align*}
	where $R=\nicefrac{T}{\tau}$ is the number of communication rounds, $\Delta^{\mathrm{ae}}_T = f^{\mathrm{ae}}(\{\vct{\theta}_{i,0}\}_i, \vct{\mu}_0, \sigma_0) - f^{\mathrm{ae}}(\{\vct{\theta}_{i,T}\}_i, \vct{\mu}_T, \sigma_T)$, and the constants can be found in lemma~\ref{lem:ae_lipschitz}.
\end{theorem}

\begin{remark}\label{remark:ae}
	By examining the multiplicative constants within the bounds specified in \cref{thm:PCA_convergence,thm:AE_convergence}, we note a consistent observation: $\sigma$ exhibits an inverse relationship with convergence speed. A higher $\sigma$, whether resulting from a large value of $\xi$ or inherent heterogeneity in the setting, can expedite the convergence process. Essentially, a large $\sigma$ diminishes collaboration, allowing the model to fit quickly due to a reduced effective number of samples. Conversely, a smaller $\sigma$ promotes collaboration and may augment the effective sample count. Observe that faster convergence does not necessarily imply a superior generalization error. Consequently, in our experiments, opting for a high value of $\sigma_0$ facilitates fast convergence in the initial stages, while still allowing flexibility for adjustments to yield a superior generalization error. 
\end{remark}
\subsubsection{Proof outline of Theorem~\ref{thm:AE_convergence}}
Here we present a proof outline for Theorem~\ref{thm:AE_convergence}, detailed proof of lemmas and intermediate steps is provided in \cref{app:ae}. We start with an assumption that is necessary to have smoothness w.r.t. $\sigma$, and a lemma indicating Lipschitz properties of the loss function. The assumption is equivalent to having an upper bound over the gradient w.r.t. $\sigma$, which is a standard assumption for model parameters.
\begin{assumption}
Assume that there exists some $B>0$ such that for the weights of the autoencoders, we have $\lVert \vct{\mu} \rVert \leq B$ and $\lVert \vct{\theta}_i \rVert \leq B$ for all $i \in [m]$.
\end{assumption}
Now we show the Lipschitzness properties of the loss function.
\begin{lemma}
\label{lem:ae_lipschitz}
The loss function $f_i^{\mathrm{ae}}(\vct{\theta}_i, \vct{\mu}, \sigma)$ is $L_{\vct{\mu}}$-smooth w.r.t. $\vct{\mu}$ and $L_\sigma$-smooth w.r.t. $\sigma$ with
\begin{gather*}
L_{\vct{\mu}} = \frac{d_{\theta}}{2 \xi \omega^2} \\
L_\sigma = \frac{3 \xi d_{\theta}^2}{2 \xi^2 \omega^4} + \frac{3 d_{\theta}^2 B^2}{\xi^2 \omega^4} + \frac{d_{\theta}^2}{2 \xi \omega^2}.
\end{gather*}
Also, we have
\begin{align*}
\left\Vert \nabla_{\vct{\mu}} f_i^{\mathrm{ae}}(\vct{\theta}, \vct{\mu}, \sigma_1) - \nabla_{\vct{\mu}} f_i^{\mathrm{ae}}(\vct{\theta}, \vct{\mu}, \sigma_2) \right\Vert
\leq L_\sigma^{(\vct{\mu})} \lvert \sigma_1 - \sigma_2 \rvert
\end{align*}
with
\begin{equation*}
L_\sigma^{(\vct{\mu})} = \frac{B \sqrt{d_{\theta}^3}}{\omega^3 \sqrt{2 \xi^3}}.
\end{equation*}
\end{lemma}

\paragraph{Sufficient decrease when $\tau$ divides $t$} At communication round time steps we have the following decrease property,
\begin{align*}
f^{\mathrm{ae}}(\{\vct{\theta}_{i,t}\}, \vct{\mu}_{t}, \sigma_{t}) - f^{\mathrm{ae}}(\{\vct{\theta}_{i,t-1}\}, \vct{\mu}_{t-1}, \sigma_{t-1})	
&\leq \left( -\eta_1+ \eta_1^2\frac{L_{\vct{\theta}}}{2} \right) \left( \frac{1}{m} \sum_{i=1}^m \left\lVert \nabla_{\vct{\theta}_{i,t-1}} f_i^{\mathrm{ae}}(\vct{\theta}_{i,t-1}, \vct{\mu}_{t-1}, \sigma_{t-1}) \right\rVert^2 \right) \\
 & \quad + \left( -\eta_2+\eta_2^2 \frac{L_\sigma + {L_{\sigma}^{(\vct{\mu})}}^2}{2} \right) (g^{\sigma}_t)^2 + \left( -\frac{\eta_3}{2} +\frac{L_{\vct{\mu}} \eta_3^2}{2} \right) \|\bg_t^{\vct{\mu}}\|^2. \numberthis \label{eq:ae_suff_dec_1}
\end{align*}

\paragraph{Sufficient decrease when $\tau$ does not divide $t$} At time steps that are not communication rounds we simply have decreased due to the updates of $\{\vct{\theta}_i\}$, that is,
\begin{align*}
f^{\mathrm{ae}}(\{\vct{\theta}_{i,t}\}, \vct{\mu}_t, \sigma_t) - f^{\mathrm{ae}}(\{\vct{\theta}_{i,t-1}\}, \vct{\mu}_{t-1}, \sigma_{t-1})
&= f^{\mathrm{ae}}(\{\vct{\theta}_{i,t}\}, \vct{\mu}_{t-1}, \sigma_{t-1}) - f^{\mathrm{ae}}(\{\vct{\theta}_{i,t-1}\}, \vct{\mu}_{t-1}, \sigma_{t-1}) \\
&= \frac{1}{m} \sum_{i=1}^m(f_i^{\mathrm{ae}}(\vct{\theta}_{i,t}, \vct{\mu}_{t-1}, \sigma_{t-1}) - f_i^{\mathrm{ae}}(\vct{\theta}_{i,t-1}, \vct{\mu}_{t-1}, \sigma_{t-1})) \\
&\leq \left( -\eta_1+ \eta_1^2\frac{L_{\vct{\theta}}}{2} \right) \left( \frac{1}{m} \sum_{i=1}^m \left\lVert \nabla_{\vct{\theta}_{i,t-1}} f_i^{\mathrm{ae}}(\vct{\theta}_{i,t-1}, \vct{\mu}_{t-1}, \sigma_{t-1}) \right\rVert^2 \right).
\end{align*}

By choosing, $\eta_1 = \frac{1}{L_{\vct{\theta}}}, \eta_2 = \frac{1}{L_{\sigma}+{L_{\sigma}^{(\vct{\mu})}}^2}, \eta_3= \min\{1,\frac{1}{L_{\vct{\mu}}}\}$, and by averaging over time steps while combining two type of decrease, we obtain the final bound.

\textit{Proof outline of Theorem~\ref{thm:AE_convergence}.} Since the form of the adaptation in the loss function is similar to the PCA loss function, the lower bound on sigma (\cref{lem:sigma_lb}) holds for \texttt{ADEPT-AE} as well. We utilize the lower bound to derive the Lipschitz smoothness constants with respect to $\vct{\mu}$ and $\vct{\sigma}$ (\cref{lem:ae_lipschitz}), and the Lipschitz smoothness constant with respect to $\vct{\theta}$ is stated in~\cref{assum:ae_theta_lip}. Then, we derive the sufficient decrease with respect to $\vct{\theta}$, $\vct{\mu}$, and $\sigma$ with the Lipschitz constants. However, we have multiple local iterations for one communication round in \texttt{ADEPT-AE}. Thus, we have to deal with the sufficient decrease separately depending on whether the round is a communication round. With careful derivation under the two cases, we can combine them in the end and get to our~\cref{thm:AE_convergence}.

\begin{remark}
In the experiments for \texttt{ADEPT-AE}, we treat sigma as a vector $\vct{\sigma} \in \mathbb{R}^{d_{\theta}}$ so that each weight in the models can learn its own $\sigma_j$ instead of sharing one $\sigma$ across all the weights. The convergence for the modified algorithm is almost identical to the proof here. Moreover, it can be shown that for the Lipschitz smoothness constant $L_\sigma$, the dependence on the dimension in the numerators becomes $d_{\theta}$ instead of $d_{\theta}^2$. This is because our lower bound in~\cref{lem:sigma_lb} will not depend on $d_{\theta}$ in this case.
\end{remark}

\section{Personalized Generation through Adaptive Diffusion Models: \texttt{ADEPT-DGM}} \label{sec:diffusion}
\label{sec:DiffModel}
\begin{algorithm}[ht]
	\caption{Personalized Adaptive Diffusion Model: \texttt{ADEPT-DGM}}
	{\bf Input:} Number of iterations $T$, learning rates ($\eta_{2}, \eta_{1}, \eta_3$), number of local iterations $\tau$, sample corruption range $\gamma \in \mathbb{Z}^+$\\
	\vspace{-0.3cm}
	\begin{algorithmic}[1]	 \label{algo:diff-gd}
		\STATE \textbf{Init} local models $\{\vct{\theta}_{i,0}\}_{i=1}^m$, global model $\vct{\mu}_0$, and $\sigma_0$.
		\STATE \textbf{On server:}
		\STATE Broadcast $\vct{\mu}_{0}$, $\sigma_{0}$ to all clients
		\FOR{$t=1$ \textbf{to} $T$}
		\STATE \textbf{On Clients:}
		\FOR {$i=1$ \textbf{to} $m$}
		\IF{$\tau$ divides $t-1$}
		\STATE Receive $\vct{\mu}_{t-1}$, $\sigma_{t-1}$
		%		\STATE $\sigma_{i,t} = \sigma_{t-1} - \eta_{2} \pt{\sigma_{t-1}} f_i(\vct{\theta}_{i, t-1}, \vct{\mu}_{t-1}, \sigma_{t-1})$
		\ENDIF
		\STATE Sample noise amount $\alpha_i \in \text{Uniform}[1,\ldots,\gamma]$ independently for each sample and construct $\vct{\alpha} = [\alpha_1, \ldots, \alpha_n]$
		\STATE $\vct{\theta}_{i,t} = \vct{\theta}_{i,t-1} - \eta_{1} \nabla_{\vct{\theta}_{i,t-1}} f^{\text{df}}_{i,\vct{\alpha}}(\vct{\theta}_{i,t-1}, \vct{\mu}_{t-1}, \sigma_{t-1})$
		\IF{$\tau$ divides $t$}
		\STATE $\vct{\mu}_{i,t} = \vct{\mu}_{t-1} - \eta_{2} \nabla_{\vct{\mu}_{t-1}} f^{\text{df}}_{i,\vct{\alpha}}(\vct{\theta}_{i,t}, \vct{\mu}_{t-1}, \sigma_{t-1})$
		\STATE $\sigma_{i,t} = \sigma_{t-1} - \eta_{3} \pt{\sigma_{t-1}} f^{\text{df}}_{i,\vct{\alpha}}(\vct{\theta}_{i, t}, \vct{\mu}_{t-1}, \sigma_{t-1})$
		\STATE Send $\vct{\mu}_{i,t}, \sigma_{i,t}$ to server
		\ELSE
		\STATE $\vct{\mu}_{t} = \vct{\mu}_{t-1}$, $\sigma_{t} = \sigma_{t-1}$
		\ENDIF
		\ENDFOR
		\STATE \textbf{At the Server:}
		\IF{$\tau$ divides $t$}
		\STATE Receive $\{ \vct{\mu}_{i,t} \}_{i=1}^m$ and $\{ \sigma_{i,t} \}_{i=1}^m$
		\STATE $\vct{\mu}_{t} = \frac{1}{m} \sum_{i=1}^m \vct{\mu}_{i,t}$, $\sigma_{t} = \frac{1}{m} \sum_{i=1}^m \sigma_{i,t}$
		\STATE Broadcast $\vct{\mu}_{t}$, $\sigma_{t}$ to all clients
		\ENDIF
		\ENDFOR
	\end{algorithmic}
	{\bf Output:} Personalized autoencoders $\{ \vct{\theta}_{1,T}, \ldots, \vct{\theta}_{m,T} \}$.
\end{algorithm}

In Algorithm~\ref{algo:diff-gd}, the main change compared to \texttt{ADEPT-AE} is that we input a corrupted sample to the network during the forward pass (\textbf{line 11}) to train it as a denoiser. Accordingly, $f^{\text{df}}_{i,\vct{\alpha}}(\vct{\theta}_i, \vct{\mu}, \sigma):= \|\vct{\phi}(\mat{X}(1-\vct{\alpha})+\mat{Z}\vct{\alpha} ; \vct{\theta}_i, \vct{\alpha}) - \mat{X}\|^2 + \frac{2 \xi + \| \vct{\mu} - \vct{\theta}_i\|^2}{2\sigma^2} + d \log \sigma$, where $\vct{\alpha}$ is the randomly sampled noise amount.
%Overall a way to is that we input a corrupted sample to the network during forward pass (\textbf{line 11}) to train it as a denoiser. 

\begin{remark}
	It is easy to see that one can prove a convergence result identical to  \cref{thm:AE_convergence} for \texttt{ADEPT-DGM} \cref{algo:diff-gd}.
\end{remark}

%\paragraph{The denoising diffusion model basics}
%The denoising diffusion model has attracted a lot of attention recently due to its capability of generating high-quality images and its theoretical foundation of the stochastic differential equation. The mathematical model of a diffusion model is the following stochastic differential equation \cite{song2020score}
%\begin{align}
%	d X_t = d W_t \quad \text{(variance exploding process)}.
%\end{align}
%For time $t \in [0, T]$, its time-reversed process is
%\begin{align}
%	d X_t^{\leftarrow} = \nabla \ln p_{X_{T-t}}(X_t^{\leftarrow}) dt + d W^{\leftarrow}_t, 
%\end{align}
%where $p_{X_t}$ is the probability density function of $X_t$, $W^{\leftarrow}_t$ is a standard Wiener process. It can be verified that $X^{\leftarrow}_t$ follows the same distribution as $X_{T- t}$. More strongly, suppose $X^{\leftarrow}_0$ has the same distribution as $X_{T}$, and then the two processes $\{X_{T - t}\}_{t \in [0, T]}$ and $\{X^{\leftarrow}_t\}_{t \in [0, T]}$ have the same distribution. 
%
%Denoise diffusion model approximates the the score function $\nabla \ln p_{X_{t}}(x)$ by a neural network $\vct{\phi}(x; \theta, t) \in \mathbb{R}^{d}$. The learning is done by minimizing the following loss function
%\begin{align}
%	\int_{t = 0}^T \mathbb{E}_{X_0, X_t}[ \| \vct{\phi}(X_t; \theta, t) - \nabla_{X_t} \log p_{X_t|X_0}(X_t | X_0) \|^2 ] d t,
%\end{align}
%where the score function $\nabla_{X_t} \log p_{t|0}(X_t | X_0)$ can be explicitly calculated when it is Gaussian distribution. 

As an illustration of the effectiveness of personalized diffusion model learning, we analyze a simple personalized Gaussian distribution generation problem, i.e., client-$i$'s target distribution is a Gaussian distribution $p(\vct{x} | \vct{\theta}_i) = \mathcal{N}(\vct{x}; \vct{\theta}_i, \sigma_0^2 \bI_d)$. In the following, we first introduce some details of the canonical denoising diffusion model for Gaussian target distribution and analyze the improvement of collaboration in the proposed personalized algorithm. 
% , specifically when the target distribution is also a Gaussian for theoretical analysis.
\vspace{-6pt}
\paragraph{Diffusion model with Gaussian target distribution:}
When the desired distribution is $\mathcal{N}(\vct{x}; \vct{\theta}, \sigma_0^2 I_d )$, the diffusion process (\ref{eqn:diffusion-ve}) is a Gaussian process and the drift term for the reverse-time process (\ref{eqn:reverse-time}) is a linear function, i.e., $\nabla \ln p_{\vct{x}_{T-t}}(\vct{x}_t^{\leftarrow})  = - \frac{\vct{x}^{\leftarrow}_t - \vct{\theta}}{\sigma_0^2 + t} $. The time-reversed process is
\begin{small}
\begin{align}
	&d \vct{x}_t^{\leftarrow} = \nabla \ln p_{\vct{x}_{T-t}}(\vct{x}_t^{\leftarrow}) dt + d \vct{w}^{\leftarrow}_t,  \vct{x}_0^{\leftarrow} \sim \mathcal{N}(\vct{\theta}, (\sigma_0^2 +T) \bI_d ).
\end{align}
\end{small}
Without the knowledge of $\vct{\theta}$, we approximate the score function by a neural network of $\vct{\phi}(\vct{x}; \hat{\vct{\theta}}, t) = - \frac{ \vct{x} - \hat{\vct{\theta}}}{\sigma_0^2 + t} $ and approximate the initial distribution of the time-reversed process $\mathcal{N}(\vct{\theta}, (\sigma_0^2 +T) \bI_d )$ by $\mathcal{N}(\vct{0}, (\sigma_0^2 +T) \bI_d )$. The learned/approximated time-reversed process is then 
\vspace{-4pt}
\begin{small}
\begin{align}
	&d \vct{x}_t^{\leftarrow} = - \frac{ \vct{x} - \hat{\vct{\theta}}}{\sigma_0^2 + t}  dt + d \vct{w}^{\leftarrow}_t , \vct{x}_0^{\leftarrow} \sim \mathcal{N}(\vct{0}, (\sigma_0^2 +T) \bI_d ). \label{eqn:learned-reversed-time}
\end{align}
\end{small}
The following lemma characterizes the difference between the generation distribution and the target distribution. 
\begin{lemma} \label{lem:ddpm-Gaussian-error}
The output distribution $p_{\vct{x}^{\leftarrow}_T | \hat{\vct{\theta}}}$ of the learned reversed-time process \eqref{eqn:learned-reversed-time} satisfies, 
\begin{align}
		D_{KL}( p_{\vct{x} | \vct{\theta}} || p_{\vct{x}^{\leftarrow}_T | \hat{\vct{\theta}}} ) = \left\| \vct{\theta} - \hat{\vct{\theta}} + \frac{\sigma_0^2}{\sigma_0^2 + T} \hat{\vct{\theta}} \right\|^2,
\end{align}
where $D_{KL}$ is the Kullback-Leibler divergence and $p_{\vct{x} | \vct{\theta}} = \mathcal{N}(\vct{x}; \vct{\theta}_i, \sigma_0^2 \bI_d)$ is the target distribution.
\end{lemma}
\vspace{-6pt}
The lemma shows that the KL-divergence between target distribution and the learned distribution is measured by the difference between $\vct{\theta}$ and $\hat{\vct{\theta}}$ a bias term $\frac{\sigma_0^2}{\sigma_0^2 + T} \hat{\vct{\theta}}$ due to the initial distribution mismatch of the time-reversed process. 

It is straightforward to verify that for $\vct{\phi}(\vct{x}; \hat{\vct{\theta}}, t) = - \frac{ \vct{x} - \hat{\vct{\theta}}}{\sigma_0^2 + t} $, training with dataset $\vct{X} = (\vct{x}_1, \ldots, \vct{x}_n)$ by maximizing ELBO \eqref{eqn:ELBO-def} has a closed-form sample-mean solution, i.e., $\hat{\vct{\theta}} = \argmax_{\vct{\theta}} ELBO_{\theta}(\vct{X}) = \frac{\sum_{j=  1}^n \vct{x}_j}{n}$. Since the data are i.i.d. sampled from $\mathcal{N}(\vct{\theta}, \sigma_0^2 \bI_d)$, we have $\mathbb{E}[D_{KL}( p_{\vct{x} | \vct{\theta}} || p_{\vct{x}^{\leftarrow}_T | \hat{\vct{\theta}}} )] = \frac{\sigma_0^2}{n}$, when omitting the initial distribution bias by $T\rightarrow \infty$. It can be viewed as personalized training without collaboration and in the following, we show that the proposed personalized training with collaboration can improve over this.

\vspace{-12pt}
\paragraph{Personalized denoising diffusion model with Gaussian targets:}
The local dataset $\mat{X}_i$ of client-$i$ sampled i.i.d. from a target Gaussian distribution $P(\vct{x} | \vct{\theta}_i) = \mathcal{N}(\vct{x}; \vct{\theta}_i, \sigma_0^2 \bI_d)$, and the population distribution is also Gaussian with $P(\vct{\theta} |{\Gamma_*}) = \mathcal{N}(\vct{\theta}; \vct{\mu}_*, \sigma_*^2 \bI_d)$. Note that we conduct analysis for a fixed unknown population distribution parameterized by $\Gamma_* = (\vct{\mu}_*, \sigma_*)$, though the proposed loss function and algorithm follows a Hierarchical Bayesian model as in Section \ref{sec:prob-form-ddpm}. 

\begin{lemma}[Personalized estimation] \label{lem:ddpm-solution}
For the parameterized score function $\vct{\vct{\phi}}(\vct{x}; \vct{\theta}, t) = - \frac{\vct{x} - \vct{\theta}}{\sigma_0^2 + t}$, the optimal solution to (\ref{eqn:loss-df}) is
\begin{align*}
\hat{\vct{\mu}} & = \frac{ \sum_{i = 1}^m \sum_{j = 1}^n \vct{x}_{ij} }{mn} \\
\hat{\vct{\theta}}_i &= \frac{n \alpha \hat{\sigma}^2}{n \alpha \hat{\sigma}^2 + 1} \frac{\sum_{j = 1}^n \vct{x}_{ij}}{n} + \frac{ \vct{\hat{\mu}} }{n \alpha \hat{\sigma}^2 + 1}, 
\end{align*}
where $\alpha = \frac{1}{\sigma_0^2} - \frac{1}{ \sigma_0^2 + T}$ and $\hat{\sigma}^2$ satisfies $\hat{\sigma}^2 = \frac{2 \xi}{d} + s^2 (\frac{ n \alpha \hat{\sigma}^2 }{n \alpha \hat{\sigma}^2 + 1})^2$ with $s^2 = \frac{\sum_{i = 1}^m  \left\| \hat{\vct{\mu}} - \frac{1}{n}\sum_{j = 1}^n \vct{x}_{ij} \right\|^2}{m d}$. 
%$\hat{\sigma}^2$ is the solution of $\hat{\sigma}^2 = \frac{2 \xi}{d} + \frac{(\frac{ n \alpha \hat{\sigma}^2 }{n \alpha \hat{\sigma}^2 + 1})^2}{d} \frac{\sum_{i = 1}^m  \left\| \hat{\vct{\mu}} - \frac{1}{n}\sum_{j = 1}^n X_{ij} \right\|^2}{m}$. 
%$\beta = \frac{ \sigma_0^2 T }{2 \sigma^2 (\sigma_0^2 + T)}$. 
%Denote by $\alpha = \frac{ \sigma_0^2 T }{2 \sigma^2 (\sigma_0^2 + T)}$ and the optimal solution to \eqref{eq:diff problem} is, 
%	\begin{align}
%		\hat{\theta}_j = \frac{\sum_{i = 1}^n x_j^i}{n} + \frac{\alpha}{n + \alpha} \left( \hat{\theta} - \frac{\sum_{i = 1}^n x_j^i}{n} \right) , \quad j = 1,\ldots, m.
%	\end{align}
%where $\hat{\theta} = \frac{\sum_{j = 1}^m \sum_{i = 1}^n x_j^i}{n m} $
\end{lemma}
\begin{remark}
	We note that the global optimum model is the average of average samples of each client. The personalized optimum model is interpolation of local estimate and the true global model. The interpolation coefficient depends on the heterogeneity (large $\sigma$ skews the result towards the local estimate and low $\sigma$ skews it towards true global model). A large amount of local samples $n$ decreases the reliance on $\vct{\mu}$ These observations are in parallel with findings in \cite{ozkara2023} on mean estimation. 
\end{remark}

\begin{theorem}[Condition for performance improvement] \label{thm:ddpm-improvement}
Consider an asymptotic regime that $m \rightarrow \infty$ and $T \rightarrow \infty$. Compared to training without collaboration, the solution of $(\{\hat{\vct{\theta}}_i\}, \hat{\vct{\mu}}, \hat{\sigma}^2)$ in Lemma \ref{lem:ddpm-solution} improves the averaged KL-divergence $\frac{1}{m} \sum_{i = 1}^m D_{KL}(p_{\vct{x} | \vct{\vct{\theta}}_i} || p_{\vct{x}^{\leftarrow}_T | \hat{\vct{\vct{\theta}}}_i}) $ by a factor of $ \left( \frac{2\hat{\sigma}^2 + \sigma_0^2/n - \sigma_*^2}{\hat{\sigma}^2 + \sigma_0^2/n}\right) \frac{\sigma_0^2/n}{\hat{\sigma}^2 + \sigma_0^2/n} \frac{\sigma_0^2}{n}$, when $\hat{\sigma}^2 >  \frac{\sigma_*^2}{2} -\frac{\sigma_0^2}{2n} $. 
\end{theorem}

\begin{corollary}\label{cor:ddpm-improvement}
Under the same setting as in Theorem \ref{thm:ddpm-improvement}, choosing $\xi \geq \frac{3d \sigma_0^2}{2n}$ guarantees strict improvement of collaboration for any population distribution $\mathcal{N}(\vct{\mu}_*, \sigma_* \bI_d)$. 
\end{corollary}

\begin{remark}
%\cref{thm:ddpm-improvement} certifies an intuitive phenomenon. Personalized learning and collaboration helps the most when the client has lower sample size $n$, and system heterogeneity $\sigma$ is smaller i.e. clients' distributions are similar. The improvement $ \left( \frac{2\hat{\sigma}^2 + \sigma_0^2/n - \sigma_*^2}{\hat{\sigma}^2 + \sigma_0^2/n}\right) \frac{\sigma_0^2/n}{\hat{\sigma}^2 + \sigma_0^2/n} \frac{\sigma_0^2}{n}$ achieves maximum when $\hat{\sigma}^2 = \sigma_*^2 $, i.e., the learned $\hat{\sigma}^2 = \sigma_*^2$.
Note that if our estimate $\hat{\sigma}^2$ of the population variance is accurate, \emph{i.e.,} $\hat{\sigma}^2 = \sigma_*^2$, then collaboration always improves over only using local data for personalized generation; and in fact the collaboration gain is the largest. In this case, the gain is larger when the number of local samples is relatively small. However, if the estimate is inaccurate, one could, in principle be better off without collaboration. However, by setting the hyperparameter $\xi \geq \frac{3d \sigma_0^2}{2n}$, we can ensure that our estimate $\hat{\sigma}^2\geq \frac{1}{2}\left(\sigma_*^2-\sigma_0^2/n\right)$, ensuring that collaboration is useful.
\end{remark}

\subsection{Proofs}
\begin{proof}[Proof of Lemma \ref{lem:ddpm-Gaussian-error}]
	Let $\beta_t = \frac{1}{T - t + \sigma_0^2}$. For stochastic differential equation
	\begin{align*}
		d \vct{x}^{\leftarrow}_{t} +  \beta_t (\vct{x}^{\leftarrow}_t - \hat{\vct{\theta}}) dt = d \vct{w}_t, \quad \vct{x}^{\leftarrow}_{0} \sim \mathcal{N}(0, (\sigma_0^2 + T) \bI_d), 
		% d \vct{x}^{\leftarrow}_{t} +  \frac{\vct{x}^{\leftarrow}_t - \hat{\vct{\theta}}}{T - t + \sigma_0^2} dt = d W_t, 
	\end{align*}
	we have 
	\begin{align*}
		& d (e^{ \int_{0}^t \beta_s ds} (\vct{x}^{\leftarrow}_{t} - \hat{\vct{\theta}}) ) = e^{ \int_{0}^t \beta_s ds} d \vct{x}^{\leftarrow}_t + e^{ \int_{0}^t \beta_s ds} \beta_t (\vct{x}^{\leftarrow}_{t} - \hat{\vct{\theta}}) dt \\
		& = e^{ \int_{0}^t \beta_s ds} \left( d \vct{x}^{\leftarrow}_t + \beta_t (\vct{x}^{\leftarrow}_{t} - \hat{\vct{\theta}}) dt \right) = e^{ \int_{0}^t \beta_s ds} d \vct{w}_t. 
	\end{align*}
Note that 
	\begin{align*}
		e^{\int_{0}^t \beta_s ds} = e^{\int_{0}^t \frac{1}{T - s + \sigma_0^2} ds} = e^{ \ln(T + \sigma_0^2) - \ln(T - t + \sigma_0^2)} = \frac{ T + \sigma_0^2 }{T - t + \sigma_0^2} 
	\end{align*}
	and
	\begin{align*}
		\int_{0}^T e^{ 2 \int_{0}^t \beta_s ds} dt = \int_{0}^T \frac{(T + \sigma_0^2)^2}{ ( T + \sigma_0^2 - t )^2} dt = (T + \sigma_0^2)^2 \left( \frac{1}{ \sigma_0^2} - \frac{1}{ T + \sigma_0^2} \right) = \left( 1 + \frac{T}{\sigma_0^2}\right) T. 
	\end{align*}
	It then follows that
	\begin{align*}
		e^{ \int_{0}^T \beta_s ds} (\vct{x}^{\leftarrow}_{T} - \hat{\vct{\theta}}) - (\vct{x}^{\leftarrow}_{0} - \hat{\vct{\theta}}) \sim \mathcal{N}\left( 0, \int_{0}^T e^{ 2 \int_{0}^t \beta_s ds} dt \right) = \mathcal{N}\left( 0, \left( 1 + \frac{T}{\sigma_0^2}\right) T  \right),
	\end{align*}
	and equivalently
	\begin{align*}
		\vct{x}^{\leftarrow}_{T} = \hat{\vct{\theta}} + \frac{\sigma_0^2}{\sigma_0^2 + T} (\vct{x}^{\leftarrow}_{0} - \hat{\vct{\theta}}) + \sqrt{\frac{\sigma_0^2 T}{\sigma_0^2 + T}} \vct{\epsilon},
	\end{align*}
	where $ \vct{\epsilon} \sim \mathcal{N}(0, \bI_d)$. 
	
Since $\vct{x}^{\leftarrow}_{0} \sim \mathcal{N}(0, (\sigma_0^2 + T) \bI_d)$, we have 
$p_{\vct{x}^{\leftarrow}_{T} | \hat{\vct{\theta}} }= \mathcal{N}\left( \hat{\vct{\theta}} - \frac{\sigma_0^2}{\sigma_0^2 + T} \hat{\vct{\theta}} , \sigma_0^2 \bI_d \right)$. 
The KL-divergence between the target distribution $p_{\vct{x}| \vct{\theta} } = \mathcal{N}(\vct{\theta}, \sigma_0^2 \bI_d)$ and $p_{\vct{x}^{\leftarrow}_{T} | \hat{\vct{\theta}} }$ can then be calculated as $ \left\| \vct{\theta} - \hat{\vct{\theta}} + \frac{\sigma_0^2}{\sigma_0^2 + T} \hat{\vct{\theta}} \right\|^2$.
\end{proof}

\begin{proof}[Proof of Lemma \ref{lem:ddpm-solution}]
The parameterized score function is $\vct{\phi}(x; \vct{\theta}, t) = - \frac{ \vct{x} - \vct{\theta}}{\sigma_0^2 + t}$. Note that $\nabla_{\vct{x}_t} \ln p_{\vct{x}_t | \vct{x}_0}(\vct{x}_t | \vct{x}_0) = - \frac{\vct{x}_t - \vct{x}_0}{t}$ since $\vct{x}_t | \vct{x}_0 \sim \mathcal{N}(\vct{x}_0, t \bI_d)$. The training loss of \eqref{eqn:loss-df} can then be written as
\begin{align*}
    \frac{1}{m}\sum_{i = 1}^m \sum_{j = 1}^n \int_{t = 0}^T \mathbb{E}_{\vct{\epsilon} \sim \mathcal{N}(0, \bI_d)} \left[ \frac{1}{2} \left\| \vct{\phi}(\vct{x}_{ij} + \sqrt{t} \vct{\epsilon}; \vct{\theta}_j, t) + \vct{\epsilon} / \sqrt{t} \right\|^2 \right] d t + \sum_{j = 1}^m \frac{2 \xi + \|\vct{\theta}_j - \vct{\mu}\|^2}{2 \sigma^2} + \frac{d}{2} \ln \sigma^2.
\end{align*}
Note that
\begin{align*}
    &\int_{t=0}^T \mathbb{E}_{\vct{\epsilon} \sim \mathcal{N}(0, \bI_d)} [ \| \vct{\phi}(\vct{x}_{ij} + \sqrt{t} \vct{\epsilon}; \vct{\theta}, t) + \vct{\epsilon} / \sqrt{t} \|^2 ] dt = \int_{t=0}^T \mathbb{E}_{\vct{\epsilon} \sim \mathcal{N}(0, \bI_d)} \left[ \left\| - \frac{\vct{x}_{ij} + \sqrt{t} \vct{\epsilon} - \vct{\theta}}{\sigma_0^2 + t} + \frac{\vct{\epsilon}}{\sqrt{t}} \right\|^2 \right] d t \\
    & = \int_{0}^T \mathbb{E}_{\vct{\epsilon} \sim \mathcal{N}(0, \bI_d)} \left[ \left\| \frac{\vct{\theta} - \vct{x}_{ij}}{\sigma_0^2 + t} + \frac{\sigma_0^2}{\sqrt{t}(\sigma_0^2 + t)} \vct{\epsilon} \right\|^2 \right] dt = \left( \int_{0}^T \frac{1}{ ( \sigma_0^2 + t )^2} dt \right) \| \vct{\theta} - \vct{x}_{ij} \|^2 + \text{const}.
\end{align*}
Since $\int_{0}^T \frac{1}{ ( \sigma_0^2 + t )^2} dt = \frac{1}{\sigma_0^2} - \frac{1}{ \sigma_0^2 + T}$, the optimization of minimizing the training loss is equivalent to
\begin{align*}
    \min_{\vct{\theta}_{1:m}, \vct{\theta}, \sigma^2} \quad   \frac{1}{m}\sum_{i = 1}^m \left( \sum_{j = 1}^n \frac{\alpha}{2} \|\vct{\vct{\theta}}_j - \vct{x}_{ij}\|^2 + \frac{ 2\xi + \|\vct{\vct{\theta}}_j - \vct{\vct{\mu}} \|^2 }{2 \sigma^2} \right) + \frac{d}{2} \ln \sigma^2,
\end{align*}
where $\alpha = \frac{1}{\sigma_0^2} - \frac{1}{ \sigma_0^2 + T} $. 
%\begin{align*}
%    \min_{\vct{\theta}_{1:m}, \vct{\mu}, \sigma^2} \quad \frac{1}{2}\left( \frac{1}{\sigma_0^2} - \frac{1}{ \sigma_0^2 + T} \right) \left( \sum_{j = 1}^m \sum_{i = 1}^n \|\vct{\theta}_j - x_j^i\|^2 \right) + \sum_{j = 1}^m \frac{2 \xi + \|\vct{\theta}_j - \vct{\mu}\|^2}{2 \sigma^2} + \frac{d}{2} \ln \sigma^2. 
%\end{align*}
%Denote by $\alpha = 2\left( \frac{1}{\sigma_0^2} - \frac{1}{ \sigma_0^2 + T} \right)$. 
%\begin{align*}
%    \min_{\vct{\theta}_{1:m}, \vct{\theta}, \sigma^2} \quad   \sum_{j = 1}^m \left( \sum_{i = 1}^n \frac{\alpha}{2} \|\vct{\vct{\theta}}_j - \vct{x}_{ij}\|^2 + \frac{ 2\xi + \|\vct{\vct{\theta}}_j - \vct{\vct{\mu}} \|^2 }{2 \sigma^2} \right) + \frac{d}{2} \ln \sigma^2. 
%\end{align*}
By the KKT condition that 
\begin{align*}
& \sum_{j = 1}^n \alpha (\vct{\vct{\theta}}_i - \vct{x}_{ij}) + \frac{ \vct{\vct{\theta}}_i - \vct{\vct{\mu}} }{\sigma^2} = 0, \quad \forall i = 1,2,\ldots, m,\\
& \vct{\vct{\mu}} = \frac{1}{m} \sum_{i = 1}^m \vct{\vct{\theta}}_i, \\
& - \frac{1}{m} \sum_{i = 1}^m \frac{ 2 \xi + \|\vct{\vct{\theta}}_j - \vct{\vct{\mu}} \|^2 }{ 2 \sigma^4 } + \frac{d}{2 \sigma^2} = 0,
\end{align*}
We thus have
\begin{align*}
\hat{\vct{\mu}} & = \frac{ \sum_{i = 1}^m \sum_{j = 1}^n \vct{x}_{ij} }{mn} \\
\hat{\vct{\theta}}_i &= \frac{n \alpha \hat{\sigma}^2}{n \alpha \hat{\sigma}^2 + 1} \frac{\sum_{j = 1}^n \vct{x}_{ij}}{n} + \frac{ \vct{\hat{\mu}} }{n \alpha \hat{\sigma}^2 + 1}, 
\end{align*}
where $\hat{\sigma}^2$ satisfies $\hat{\sigma}^2 = \frac{2 \xi}{d} + s^2 (\frac{ n \alpha \hat{\sigma}^2 }{n \alpha \hat{\sigma}^2 + 1})^2$ with $s^2 = \frac{\sum_{i = 1}^m  \left\| \hat{\vct{\mu}} - \frac{1}{n}\sum_{j = 1}^n \vct{x}_{ij} \right\|^2}{m d}$. 
%\begin{align*}
%\hat{\vct{\vct{\mu}}} & = \frac{ \sum_{i = 1}^m \sum_{j = 1}^n \vct{x}_{ij} }{mn}, \\
%\hat{\vct{\vct{\theta}}}_i &= \frac{\sum_{j = 1}^n \vct{x}_{ij}}{n} + \frac{ 1 }{n \alpha \sigma^2 + 1} \left( \vct{\vct{\mu}} - \frac{\sum_{j = 1}^n \vct{x}_{ij}}{n} \right), \\
%\hat{\sigma}^2 &= \frac{2 \xi}{d} + \frac{(\frac{ n \alpha \sigma^2 }{n \alpha \sigma^2 + 1})^2}{d} \frac{\sum_{i = 1}^m  \left\| \vct{\vct{\mu}} - \frac{1}{n}\sum_{j = 1}^n \vct{x}_{ij} \right\|^2}{m}. 
%\end{align*}
\end{proof}

%\begin{align}
%\frac{ 2\xi + d s^2 (\frac{ n \alpha \sigma^2 }{n \alpha \sigma^2 + 1})^2}{2 \sigma^2} + \frac{d}{2} \ln \sigma^2
%& = \frac{\xi}{\sigma^2} + d s^2 \frac{\sigma^2}{2 (\sigma^2 + 1/(n\alpha))^2} + \frac{d}{2} \ln \sigma^2 \\ 
%& =  \frac{\xi}{\sigma^2} + d s^2 \frac{1}{2 \sigma^2 + \frac{2}{(n\alpha)^2\sigma^2} + \frac{4}{n\alpha}} + \frac{d}{2} \ln \sigma^2
%\end{align}

\begin{proof}[Proof of Theorem \ref{thm:ddpm-improvement}]
When $m \rightarrow \infty$, $s^2 \rightarrow \frac{\sigma_0^2}{n} + \sigma_*^2$ and $\hat{\vct{\mu}} \rightarrow \vct{\mu}_*$, a.s..
%\rightarrow \mathcal{N}(\sigma_*^2, \frac{\sigma_*^2}{n \sqrt{m}})
$\hat{\sigma}^2$ satisfies 
\begin{align*}
\hat{\sigma}^2 = \frac{2 \xi}{d} + (\frac{\sigma_0^2}{n} + \sigma_*^2) (\frac{ \hat{\sigma}^2 }{ \hat{\sigma}^2 + 1 / (n \alpha) })^2.
\end{align*}

%$\hat{\sigma}^2$ is the solution of 
%\begin{align*}
%\hat{\sigma}^2 = \frac{2 \xi}{d} + s^2 (\frac{ \hat{\sigma}^2 }{ \hat{\sigma}^2 + 1 / (n \alpha) })^2
%\end{align*}
%When $m \rightarrow \infty$, $\hat{\sigma}^2 \rightarrow \hat{\sigma}^2$, and $\hat{\sigma}^2$ is a solution of the equation above with $s^2 = \frac{\sigma_0^2}{n} + \sigma_*^2$. 

Since $\vct{\vct{\theta}}_i$ are sampled i.i.d. from a population distribution $\mathcal{N}(\vct{\mu}_*, \sigma_*^2 \bI_d)$. $\alpha = 1/\sigma_0^2$ since $T \rightarrow \infty$. Let $\vct{x} \sim \mathcal{N}(\vct{\vct{\theta}}, \frac{\sigma_0^2}{n} \bI_d)$, $\vct{\vct{\theta}} \sim \mathcal{N}(\vct{\vct{\mu}}, \sigma^2 \bI_d)$, by Lemma \ref{lem:ddpm-Gaussian-error}, we have
\begin{align*}
\frac{1}{m} \sum_{i = 1}^m D_{KL}(p_{\vct{x} | \vct{\vct{\theta}}_i} || p_{\vct{x}^{\leftarrow}_T | \hat{\vct{\vct{\theta}}}_i}) & = \frac{1}{m} \sum_{i = 1}^m \left\| \vct{\vct{\theta}}_{i} - \hat{\vct{\vct{\theta}}}_{i} + \frac{\sigma_0^2}{\sigma_0^2 + T} \hat{\vct{\vct{\theta}}}_{i} \right\|^2 \\
& = \mathbb{E}\left[ \left\| \vct{\vct{\theta}} - \vct{x} - \frac{1}{n \alpha \hat{\sigma}^2 + 1} (\vct{\vct{\mu}} - \vct{x}) \right\|^2 \right] \\
& = (\frac{\sigma_0^2/n}{\hat{\sigma}^2 + \sigma_0^2/n})^2 \mathbb{E}\left[ \left\| \vct{\vct{\theta}} - \vct{\vct{\mu}}\right\|^2 \right] + (\frac{\hat{\sigma}^2}{\hat{\sigma}^2 + \sigma_0^2/n})^2 \mathbb{E}\left[ \left\| \vct{\vct{\theta}} - \vct{x}\right\|^2 \right] \\
& = (\frac{\sigma_0^2/n}{\hat{\sigma}^2 + \sigma_0^2/n})^2 \sigma_*^2 + (\frac{\hat{\sigma}^2}{\hat{\sigma}^2 + \sigma_0^2/n})^2 \sigma_0^2/n \\
& = \frac{\sigma_0^2}{n} + \frac{(\sigma_*^2 + \sigma_0^2/n) \sigma_0^2/n}{(\hat{\sigma}^2 + \sigma_0^2/n)^2} \frac{\sigma_0^2}{n} - 2 \frac{\sigma_0^2/n}{\hat{\sigma}^2 + \sigma_0^2/n} \frac{\sigma_0^2}{n} \\
% & = \frac{\sigma_0^2}{n} - \frac{\sigma_0^2/n}{\hat{\sigma}^2 + \sigma_0^2/n} \frac{\sigma_0^2}{n} - \frac{\hat{\sigma}^2 - \sigma_*^2}{\hat{\sigma}^2 + \sigma_0^2/n}\frac{ \sigma_0^2/n}{\hat{\sigma}^2 + \sigma_0^2/n} \frac{\sigma_0^2}{n} \\
& = \frac{\sigma_0^2}{n}  - \left( \frac{2\hat{\sigma}^2 + \sigma_0^2/n - \sigma_*^2}{\hat{\sigma}^2 + \sigma_0^2/n}\right) \frac{\sigma_0^2/n}{\hat{\sigma}^2 + \sigma_0^2/n} \frac{\sigma_0^2}{n}
\end{align*}
where the expectation is taken w.r.t. $\vct{x}, \vct{\vct{\theta}}$. 

Since without collaboration, the training of maximizing $ELBO_i$ for client-$i$ leads to parameter $\hat{\vct{\vct{\theta}}}_i = \frac{\sum_{j = 1}^n \vct{x}_{ij}}{n}$ and the KL-divergence between the target distribution and the output distribution is $\frac{\sigma_0^2}{n}$, it follows that collaboration improves the performance as long as $\hat{\sigma}^2 > \frac{\sigma_*^2}{2} -\frac{\sigma_0^2}{2n} $. 

The improvement is $ \left( \frac{2\hat{\sigma}^2 + \sigma_0^2/n - \sigma_*^2}{\hat{\sigma}^2 + \sigma_0^2/n}\right) \frac{\sigma_0^2/n}{\hat{\sigma}^2 + \sigma_0^2/n} \frac{\sigma_0^2}{n}
$ and achieves the maximum when $\hat{\sigma}^2 = \sigma_*^2 $, i.e., the learned $\hat{\sigma}^2 = \sigma_*^2$
\end{proof}

\begin{proof}[Proof of Corollary \ref{cor:ddpm-improvement}]
Under the same setting as in Theorem \ref{thm:ddpm-improvement}, by Lemma \ref{lem:ddpm-solution}, we have $\hat{\sigma}^2 = \frac{2 \xi}{d} + (\frac{\sigma_0^2}{n} + \sigma_*^2) (\frac{ \hat{\sigma}^2 }{\hat{\sigma}^2 + \sigma_0^2/n})^2$. Taking $\xi > \frac{2 \xi}{d} = \frac{3d \sigma_0^2}{2n}$ gives that $\hat{\sigma}^2 \geq \frac{3 \sigma_0^2}{n}$, and thus $\hat{\sigma}^2 > (\frac{\sigma_0^2}{n} + \sigma_*^2) (\frac{ \hat{\sigma}^2 }{\hat{\sigma}^2 + \sigma_0^2/n})^2 \geq \frac{9}{16} (\frac{\sigma_0^2}{n} + \sigma_*^2) > \frac{\sigma_*^2}{2} - \frac{\sigma_0^2}{2n}$, which guarantees strictly improvement by Theorem \ref{thm:ddpm-improvement}.
\end{proof}

% \begin{align}
%     \frac{\partial}{\partial t} p_t(x) = \Delta \left( \frac{1}{2} p_t(x) \right) \quad \Leftrightarrow \quad \frac{\partial}{\partial t} \ln p_t(x) = \frac{1}{2} \nabla \ln p_t(x) \cdot \nabla \ln p_t(x) + \frac{1}{2} \Delta \ln p_t(x)
% \end{align}

% Let $q_t$ be the pdf of $X^{\leftarrow}_t$, we then have
% % \begin{align}
% %     \frac{\partial}{\partial t} \ln q_t(x) = - \frac{1}{2} \nabla \ln p_t(x) \cdot \nabla \ln q_t(x) - \frac{1}{2} \Delta \ln q_t(x)
% % \end{align}
% \begin{align}
%     \frac{\partial}{\partial t} q_t(x) = \nabla \cdot \left( \frac{x - \vct{\theta}}{\sigma_0^2 + t} q_t(x) \right) + \frac{1}{2} \Delta q_t(x)
% \end{align}

% \begin{align}
%     \frac{\partial}{\partial t} \ln q_t(x) = \nabla \cdot \left( \frac{x - \vct{\theta}}{\sigma_0^2 + t} \right) + \frac{x - \vct{\theta}}{\sigma_0^2 + t} \cdot \nabla \ln q_t(x) + \frac{1}{2} \nabla \ln q_t(x) \cdot \nabla \ln q_t(x) + \frac{1}{2} \Delta \ln q_t(x)
% \end{align}

% $\ln q_0(x) = -\frac{ x^2 }{2(\sigma_0^2 + T)} - \frac{1}{2} \ln(2\pi (\sigma_0^2 + T) )$  

% Check $\ln q_t(x) = -\frac{ (x - \theta_{t})^2 }{2(\sigma_0^2 + T - t)} - \frac{1}{2} \ln(2\pi (\sigma_0^2 + T - t) )$ for some $\theta_t$ with $\theta_0 = 0$. 

\section{Experiments}
\label{sec:Expts}

\subsection{Experimental Setting}
   
In our experiments, our goal is to compare our adaptive personalized unsupervised algorithms with global training (FedAvg, FedAvg+fine-tuning), local training (training individual models without collaboration), and competitive baselines in terms of testing performance under different heterogeneous scenarios. For all experiments, we use 50 clients ($m=50$) and initialize $\xi=1e{-}6$. 
  
\paragraph{\texttt{ADEPT-PCA}:}
We use synthetic datasets for the experiments of \texttt{ADEPT-PCA:}. In the dataset, we first sample a global PC, $\bV^* \in St(d,r)$, uniformly on the Steifel manifold. We sample $\{\hat{\bU}_i^*\}_{i=1}^m$ where the entries of each $\hat{\bU}_i^*$ follows Gaussian distribution with mean being $\bV^*$ and variance $\sigma^*$. Then, we let $\bU_i^* = \calR_{\bV^*}(P_{\mathcal{T}_{\bV^*}} (\hat{\bU}_i^*))$ so that it is in the Steifel manifold. Data on each client are then generated by $\bx = \bU_i^* \bz + \boldsymbol{\epsilon}$.
\begin{figure}[h]
	\centering
	  
	\includegraphics[width=0.5\textwidth, trim = 0 0 0 30, clip=true]{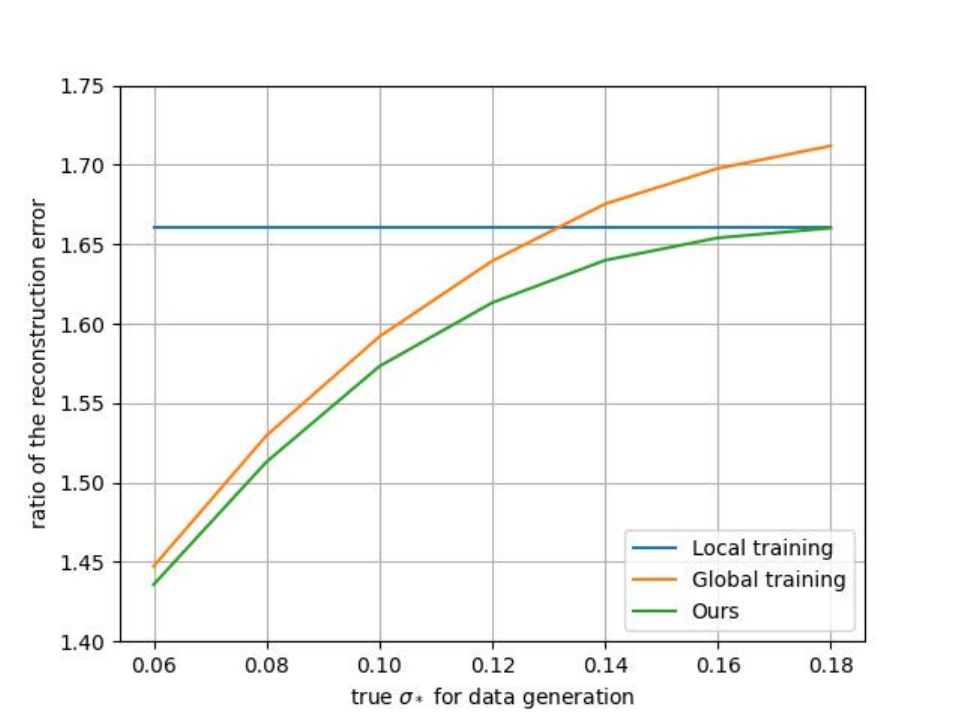}
	  
	\caption{Ratio of the reconstruction error of different methods to the true model w.r.t. different values of $\sigma^*$. We have $d=100$, $r=20$, $m=10$, and $n=20$.  }
	\label{fig:pca_exp}
\end{figure}
   
\paragraph{\texttt{ADEPT-AE:}} For AEs we do synthetic and real data experiments. In the synthetic experiments from a 0 mean $\sigma_\mu = 0.1 $ standard deviation Gaussian, we sample weights for a one-layer decoder $\mu^{d,*}$ with 5 latent dimensionality and 64 output dimensionality. Then by perturbing the weights with another zero-mean Gaussian with $\sigma^*$ we obtain true personalized decoders, which is used as in \eqref{eq:nonlin latent model} to generate 10 local samples across 50 clients. Heterogeneity among clients will depend on $\sigma^*$ and can be quantified in terms of signal-to-noise ratio as $20\log_{10}\frac{\sigma_{\mu}}{\sigma^*}$ dB. For the real data experiments, we use MNIST, Fashion MNIST, and CIFAR-10. To introduce heterogeneity, we distribute the samples such that each client has access to samples from a single class which simulates distinct data distributions for each client (commonly referred as pathological heterogeneity \cite{mcmahan2017communicationefficient}). For MNIST and Fashion MNIST, each client has access to ~120 training samples; and for CIFAR-10, 250 samples.
\begin{table}[h]
	   
	\caption{Total energy captured \% averaged over samples and clients in the synthetic experiments.    }\label{tab:ae_synth}
	\[
	\begin{array}{@{}l*{4}{c}@{}}
		\toprule
		\text{Method} & \multicolumn{3}{c@{}}{\text{Heterogeneity (std of noise and SNR)}}\\
		\cmidrule(l){2-4}
		& 0.05 (6\text{dB}) & 0.025 (12\text{dB}) & 0.01 (20\text{dB}) \\
		\midrule
		\text{Baseline} & 88.3 \pm 0.5 & 95.4 \pm 0.8 & 98.6 \pm 0.1 \\
		\texttt{ADEPT-AE} & \mathbf{87.3} \pm 1.1 & \mathbf{95.9} \pm 0.1 & \mathbf{98.7}  \pm 0.1\\
		\text{Global Training} & 81.3 \pm 0.1 & 94.3 \pm 0.2 & 98.4 \pm 0.4\\
		\text{Local Training} & 83.2 \pm 1.9& 83.2 \pm 1.9 & 83.2 \pm 1.9\\
		\bottomrule
	\end{array}
	\]
\end{table}

\textbf{Models. } For the synthetic experiments we use a two layer fully connected AE, for MNIST and Fashion MNIST we also use a two layer AE with 784 input dimension with 10 and 20 latent dimensions depending on experiment. For CIFAR-10 we use a symmetric convolutional AE whose input and output layers are convolutional layers with 16 channels, 3 kernel size, 2 stride and no padding; the intermediate layers are fully connected layers that maps 3600 dimensions to latent dimensions. We use 10 and 50 latent dimensionality depending on the experiment. We use ReLU activation function after the first layer and sigmoid after the last layer. 

\textbf{Training and hyper-parameters. } For the synthetic experiments we don't do local iterations and just do distributed training. For the other experiments, we do 20 local iterations per communication round, and every communication round corresponds to an epoch, i.e. we use 300 global batch size for MNIST and 750 for CIFAR-10. For all datasets and methods, we use SGD with a constant learning rate of $0.01$ after individually tuning in the set $\{0.5,0.1,0.05,0.01,0.005,0.001\}$ and momentum coefficient of $0.9$. For our method, we choose $\eta_2=0.01, \eta_3 = 0.001$ and use SGD without momentum. For synthetic, MNIST, and Fashion MNIST datasets we train for 150 epochs/comm. rounds and for CIFAR-10 for 250 epochs. We initialize $\sigma=1$ in MNIST and Fashion MNIST experiments, and $\sigma=0.4$ in synthetic experiments, and do not update $\sigma$ for the first two epochs. For CIFAR-10 we initialize $\sigma=0.2$ and we do lazy updates that is we start updates after 200 epochs. We observed lazy updates with relatively small initial $\sigma$ works better for deeper models, whereas simpler models do not require it. To improve the empirical performance and have a more stable training we make a few changes to \cref{algo:ae-gd}. Namely, we keep individual $\sigma$ for each scalar weight, for personalized and global models we clip the $\ell_\infty$ norm of the gradients by 1, and for $\sigma$ by 10. We update $\sigma$ at the first iteration instead of the last one. We include the global model in the local iterations.

\begin{table}[h]
	   
	\caption{Total energy captured \% averaged over samples and clients in the real dataset experiments.    }
	\label{tab:ae_real}
	\[
	\begin{array}{@{}l*{4}{c}@{}}
		\toprule
		\text{Dataset} & \text{Method} & \multicolumn{2}{c@{}}{\text{Latent dimensionality}}\\
		\cmidrule(l){3-4}
		& & Low & High \\
		\midrule
		\hline
		\textit{MNIST} & \text{Baseline} & 78.9 \pm 0.5 & 85.7 \pm 0.4  \\
		& \texttt{ADEPT-AE} & \mathbf{70.8} \pm 0.5 & \mathbf{77.7} \pm 0.1\\
		& \text{FedAvg} & 66.2 \pm 0.9 & 75.9 \pm 0.7 \\
		& \text{Local Training} & 67.0 \pm 0.8& 69.1 \pm 1.1 \\
		\hline 
		\textit{F. MNIST} &	\text{Baseline} & 88.5 \pm 0.2  & 91.3 \pm 0.1   \\
		&\texttt{ADEPT-AE} & \mathbf{83.9} \pm 0.2 & \mathbf{85.6} \pm 0.2\\
		&\text{FedAvg} & 81.2 \pm 0.2 & 84.9 \pm 0.2 \\
		&\text{Local Training} & 76.9 \pm 2.0& 77.1 \pm 0.5 \\
		\hline 
		%		\multicolumn{2}{@{}l}{\textit{CIFAR-10}}\\
		\textit{CIFAR-10}	&\text{Baseline} & 88.7 \pm 0.5  & 93.3 \pm 0.1  \\
		&\texttt{ADEPT-AE} & \mathbf{88.4} \pm 0.5 & \mathbf{93.3} \pm 0.2\\
		&\text{FedAvg} & 87.4 \pm 0.2 & 91.2 \pm 0.1 \\
		&\text{Local Training} & 87.7 \pm 0.2& 92.2 \pm 0.2 \\
		\bottomrule
	\end{array}
	\]
\end{table}
  
\paragraph{\texttt{ADEPT-DGM:}}
For diffusion experiments we use a 6-layer U-Net from \href{https://github.com/huggingface/diffusion-models-class/blob/main/unit1/02_diffusion_models_from_scratch.ipynb}{Hugging Face}. We let every client to have access to 1200 or 600 samples from one class depending on the experiment. Instead of FedAvg, we compare to FedAvg+fine-tuning, as FedAvg cannot exclusively generate samples from the client's target distribution.

\textbf{Training and hyper-parameters. } We train using 20 local iterations per communication round and epoch. We use Adam optimizer with $1e{-}3$ for all methods. For our method, we use Adam with $0.01$ learning rate for the updates of the global model and SGD with $0.001$ lr for $\sigma$. We do 100 epochs/comm. rounds in total. We initialize $\sigma=0.8$ and do not update for the first 2 epochs. We multiply the learning rates $\eta_1,\eta_2$ by 0.1 at 75th epoch. We employ the same changes in \cref{algo:ae-gd} in \cref{algo:diff-gd} as well. For a simpler demonstration, we use a variance preserving SDE (as in \cref{algo:diff-gd}) instead of variance exploding (as in \cref{sec:diffusion}). We do the same to modify Algorithm~\ref{algo:diff-gd} as we did for \texttt{ADEPT-DGM}.

\begin{figure}[t]
	\centering
	{\includegraphics[width=.25\textwidth]{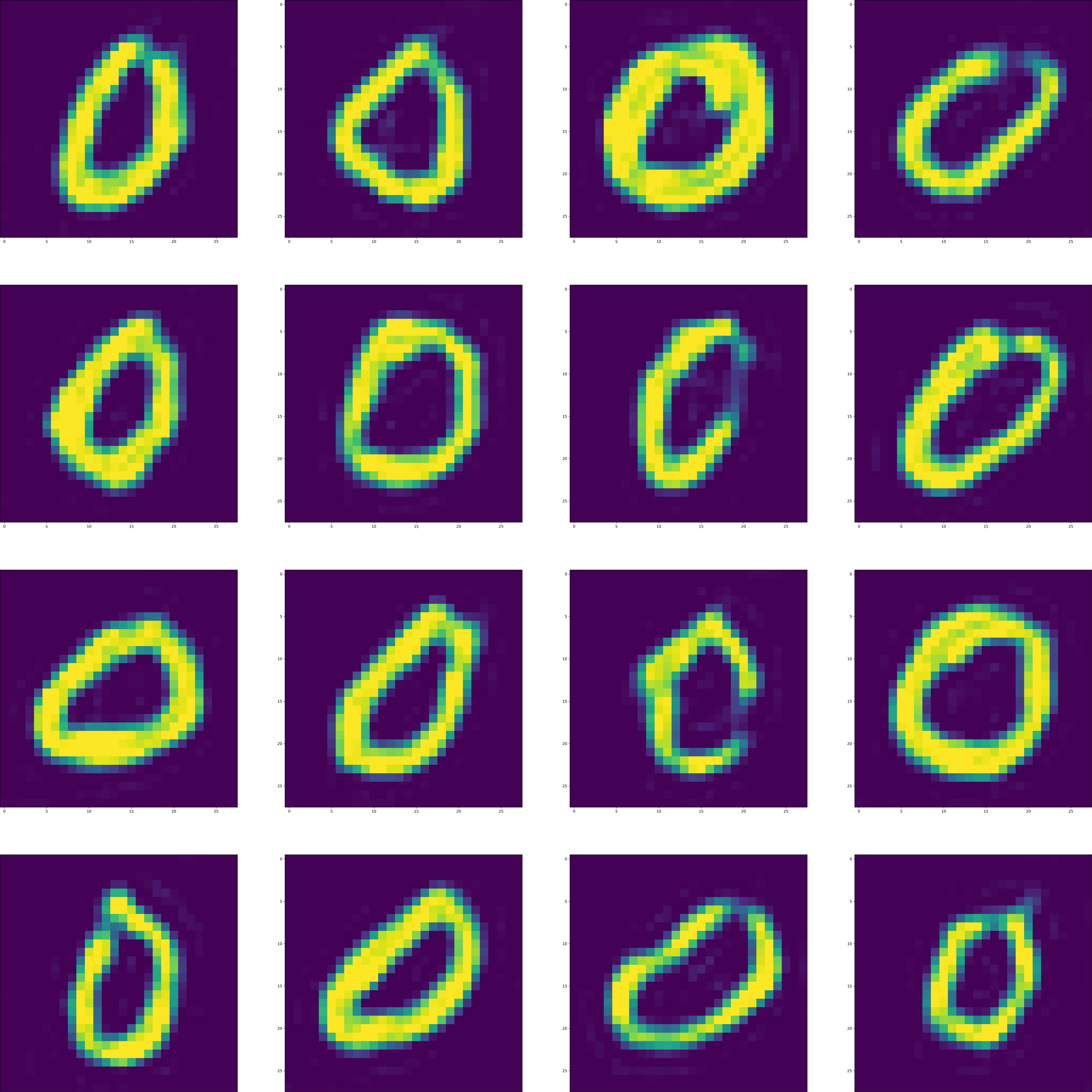}}\hfill
	{\includegraphics[width=.25\textwidth]{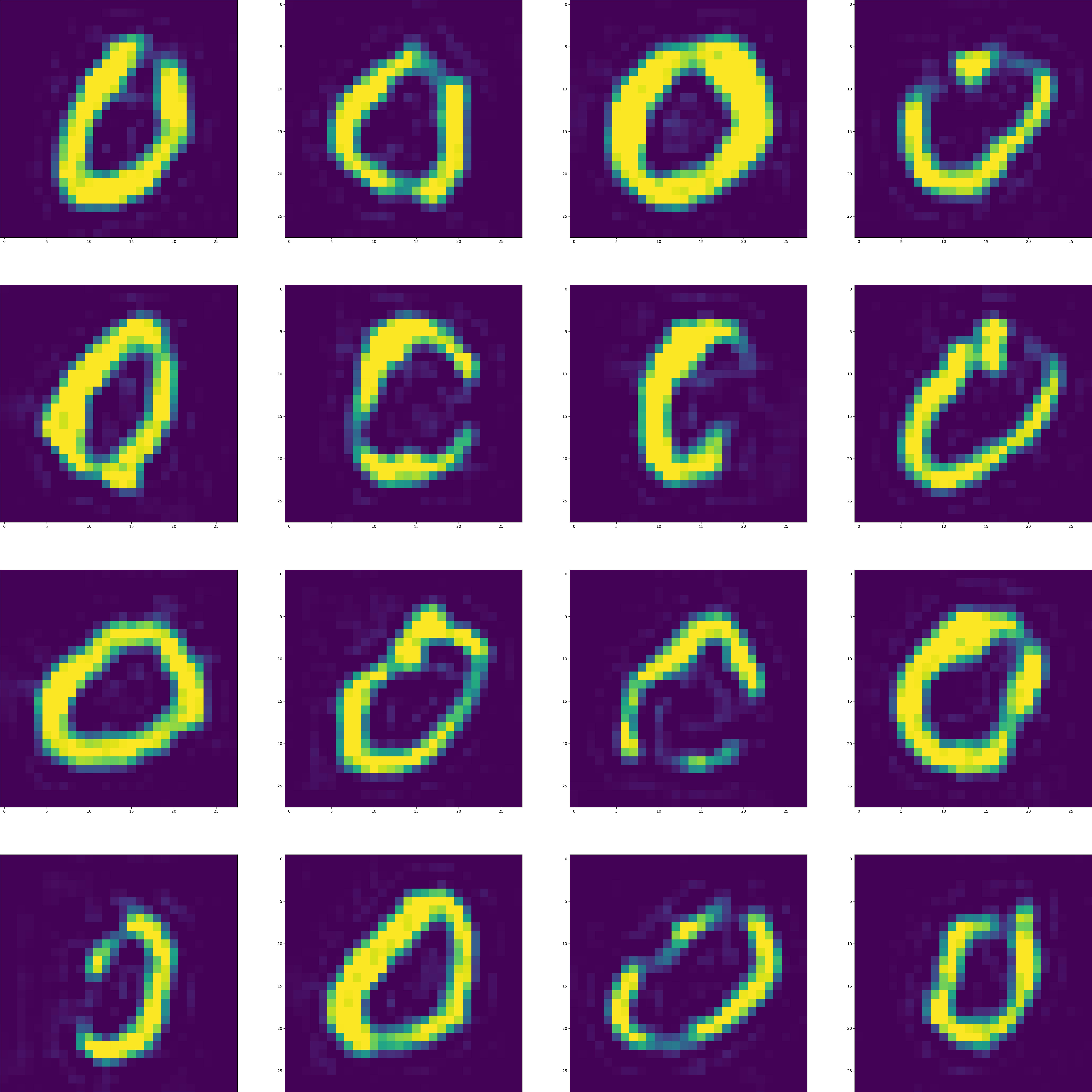}}\hfill
	{\includegraphics[width=.25\textwidth]{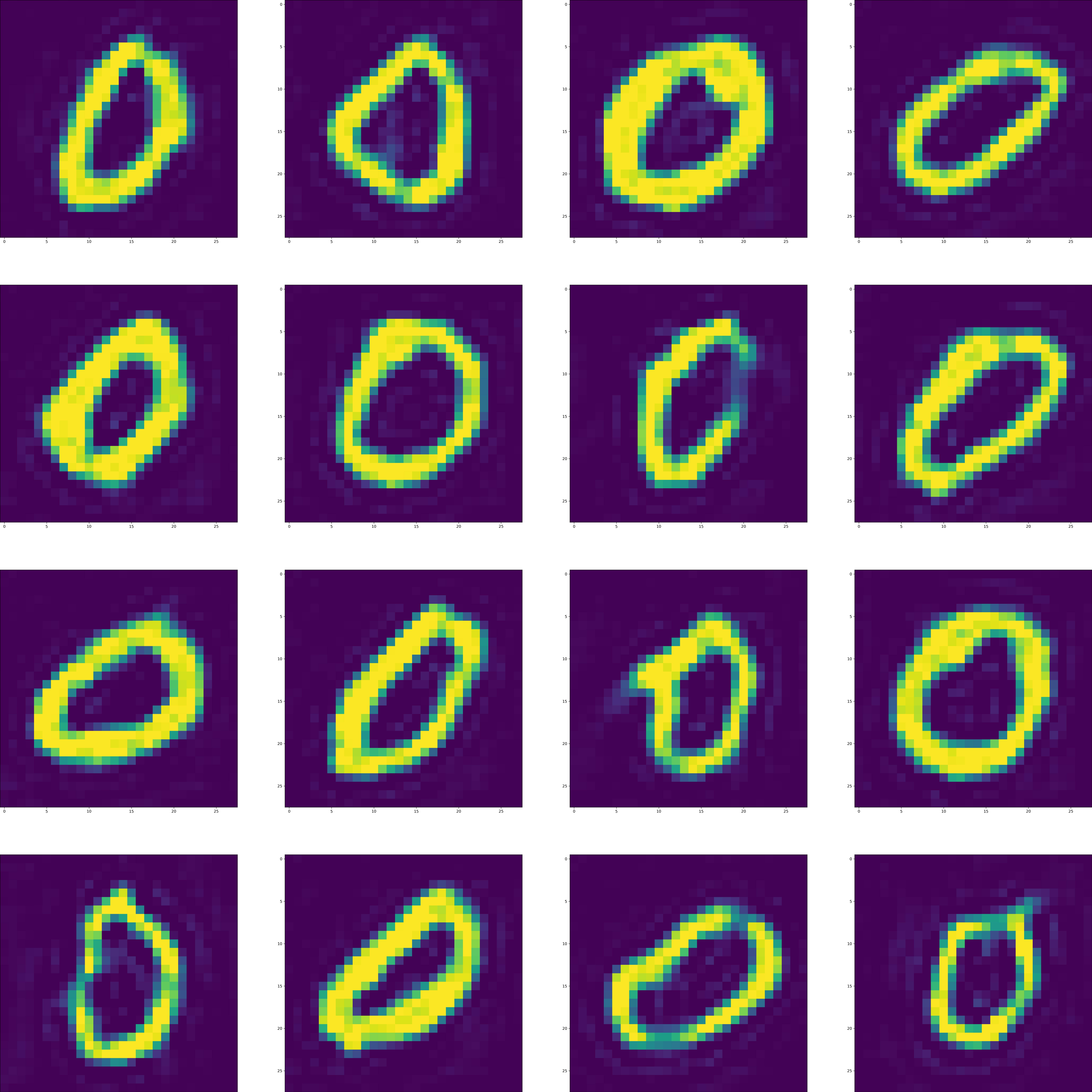}}
	%\subfigure{file}
	%	\includegraphics[width=.3\textwidth]{grid_prs.jpg}\hfill
	%	\includegraphics[width=.3\textwidth]{grid_fed.jpg}\hfill
	%	\includegraphics[width=.3\textwidth]{grid_loc.jpg}
%	 
	\caption{Randomly chosen samples (Left:\texttt{ADEPT-DGM}, noise $\sigma=0.024$; Middle:FedAvg+fine-tuning,noise $\sigma=0.028$; Right:Local training, noise $\sigma=0.032$) (models are trained and samples are chosen with the same seed across runs) from generated dataset for a client with data from '0' class.}
	\label{fig:diff_images}
\end{figure}
\subsection{Results}

\begin{table}[h]
	   
	\caption{Diffusion model generation quality for generating MNIST samples using U-Net model (lower is better).    }
	\label{tab:diff}
	\[
	\begin{array}{@{}l*{3}{c}@{}}
		\toprule
		\text{Method} & \multicolumn{2}{c@{}}{\text{Metric}}\\
		\cmidrule(l){2-3}
		& FID & KID \\
		\midrule
		\text{Baseline} & 72.3 \pm 2.2 & 0.062 \pm 0.003 \\
		\midrule
		\multicolumn{2}{@{}l}{{\textit{High Number of Samples}}} \\
		\texttt{ADEPT-DGM} & \mathbf{80.0} \pm 2.3 & \mathbf{0.067} \pm 0.003 \\
		\text{FedAvg+fine-tuning} & 88.5 \pm 3.8 & 0.082 \pm 0.004 \\
		\text{Local Training} & 84.1 \pm 0.8 & 0.075 \pm 0.001\\
		\hline 
		\multicolumn{2}{@{}l}{{\textit{Low Number of Samples}}} \\
		\texttt{ADEPT-DGM} & \mathbf{84.2} \pm 1.5 & \mathbf{0.069} \pm 0.002 \\
		\text{FedAvg+fine-tuning} & 95.9 \pm 7.6 & 0.090 \pm 0.009 \\
		\text{Local Training} & 91.8 \pm 2.0 & 0.083 \pm 0.003\\
		\bottomrule
	\end{array}
	\]
	   
\end{table}

\paragraph{\texttt{ADEPT-PCA:}}
We compare the reconstruction error between Algorithm~\ref{algo:pca}, local training, and global training. In the global training setting, we train a single global model with the average of local gradients in each iteration. In Figure~\ref{fig:pca_exp}, the value in the $y$-axis is the ratio of the reconstruction error of each training method to the true error, which is evaluated by $\{\bU_i^*\}_{i=1}^m$. When $\sigma^*$ is small, the heterogeneity of the data among the clients is low, and thus global training benefits from the sample size and performs better than pure local training. Our algorithm also makes use of the sample size and achieves an even smaller reconstruction error with the personalized models. When $\sigma^*$ is large, the heterogeneity of the data among the clients is high and thus training a single global model for each client does not work well. In this case, our algorithm learns a larger $\sigma$ and performs more like local training. In the scenario between the two cases, our algorithm also outperforms both global and local training.

\textbf{\texttt{ADEPT-AE}:}
The results are in terms of percentage of total energy captured per sample which is, for a sample $x$, equal to $100(1-\|x - \hat{x}\|^2/\|x\|^2)$. The results are averaged over 3 runs and reported together with the standard deviation.

\textbf{Synthetic data.} In Table~\ref{tab:ae_synth}, baseline denotes that each client trains a personalized AE whose decoder part is the true data generating decoder. Our method outperforms local and global training and even the competitive baseline when heterogeneity is smaller. The result is similar to Figure~\ref{fig:pca_exp}, showing our method outperforms both local and global training in the regimes where they are strong alternatives.

\textbf{Real data. } The competitive baseline in (Table~\ref{tab:ae_real}) is when 10 clients maintain all the training data from their corresponding class ($n=5000$ per client). Remarkably, our method ($n=250$ per client) matches the baseline on CIFAR-10 (for high latent dimensions), indicating that adaptive personalized collaboration results in $\times 20$ effective sample size. For other datasets, our method consistently outperforms FedAvg and local training by an important margin regardless of latent dimensionality. Our method reduces reconstruction error by as much as $\sim 35\%$ and $\sim 25\%$ compared to local training and FedAvg respectively. 
\begin{figure}[h]
	%	\begin{minipage}{.4\linewidth}
	%			\centering
	%		\includegraphics[scale=0.45]{fidviolin.pdf}
	%		   
	%		\caption{Violin plot of FID values of clients
	%			.}
	%		\label{fig:fidviolin}
	%	\end{minipage}\hfill
	%	\begin{minipage}{.4\linewidth}
	\centering
	\vspace{-0pt}
	\includegraphics[scale=0.45]{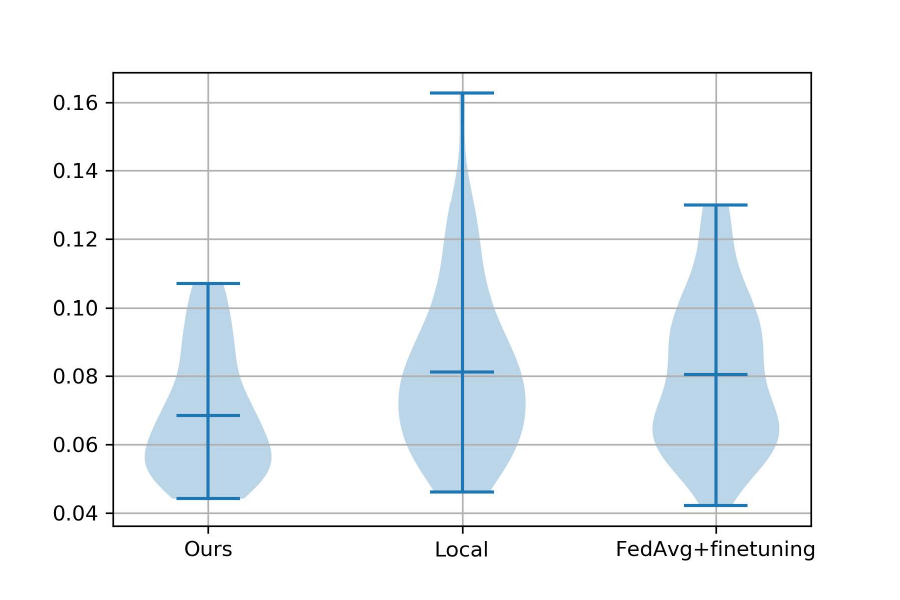}
	   
	\caption{Violin plot of KID values of clients.  }
	\label{fig:kidviolin}
	%	\end{minipage}\hfill
\end{figure}
\paragraph{\texttt{ADEPT-DGM}:}
 For Diffusion models we use FID \cite{heusel2017gans} and KID \cite{binkowski2018demystifying} metrics to quantitatively measure the generated dataset quality (see Table~\ref{tab:diff}). At the end of training, each client generates 200 samples (using the model with the lowest validation loss) and compares it to the local test dataset to compute the metrics. Our method consistently results in better quality generated samples and improves upon other methods by $ 5\% -  22\%$. Moreover, in \cref{fig:kidviolin} we depict the resulting KID values of clients. We see that our method brings equity, that is, the worst-performing client is much better compared to the worst clients of other methods; and the performance variance of clients is lower. We also illustrate randomly chosen sample images in \cref{fig:diff_images} and estimate the amount of noise using \cite{IMMERKAER1996}. Compared to \texttt{ADEPT-DGM}, in images obtained using FedAvg+fine-tuning we observe missing features and inconsistent hallucinations. On the other hand, local training outputs images with significantly more background noise as apparent in images (e.g. 1st from the last row and 2nd from the first row) and from the estimated noise standard deviation which indicates a $1.5\times$ increase in noise level in terms of noise standard deviation ($\sigma = 0.032$ for local training vs $\sigma = 0.024$ for adaptive personalized method).

%\begin{table}[h!]
%	\centering
%	\begin{tabular}{lcccc}
%		Method & & & &
%		Generative std & Baseline & Adaptive Personalized Training & FedAvg & Local Training \\
%		0.05 (6dB SNR) & 88.3 & 87.7 & 81.25 &  83.2 \\
%		0.025 (12dB SNR) & 95.4 & 96.0 & 94.3 &  83.2 \\
%		0.01 (20dB SNR) & 98.6 & 98.8 & 98.9 & 83.2
%	\end{tabular}
%	\caption{Synthetic dataset, 64 input dims, 5 latent dims, 10 training sample, total energy captured in \%}
%\end{table}

%\vspace{-6pt}
\section{Conclusion}
%We utilized a hierarchical Bayesian model of data generation for personalized unsupervised FL. In particular, we motivated loss functions, proposed novel algorithms and analyzed theoretical properties of three modalities: PCA, AEs, and denoising diffusion models. Each of our algorithms included adaptation of the heterogeneity parameter $\sigma$ during training which resulted in novel theoretical interpretations and superior empirical performance. Open questions include extensions with information constraints such as communication and privacy. 

We developed, \texttt{ADEPT}, a hierarchical Bayes framework for personalized federated unsupervised learning; leading to new criteria for linear (\texttt{ADEPT-PCA}), non-linear (\texttt{ADEPT-AE}) dimensionality reduction, and personalized federated diffusion models (\texttt{ADEPT-DGM}). Each of our algorithms included adaptation for the heterogeneity during training which resulted in novel theoretical interpretations and superior empirical performance. Open questions include extensions with information constraints such as communication and privacy. 

\bibliographystyle{plain}
\bibliography{bibliography}

\appendix
\begin{center}
{\LARGE \bf Appendix}
\end{center}
\allowdisplaybreaks{
% \section{Related Works} \label{app:related_works}
% \input{appendix_related_works}
% \section{Preliminaries} \label{app:pre}
% \input{appendix_preliminaries}
\section{Proofs for Adaptive PCA: \texttt{ADEPT-PCA}} \label{app:pca}

\begin{theorem}
	By choosing $\eta_1 = \min\{ \frac{1}{3C_{\eta_1}}, 1\}$, $\eta_2 = \min\{ \frac{1}{3C{\eta_2}}, 1\}$, and $\eta_3 = \min \Big\{ \frac{\eta_1}{3 (L_U^{(\sigma)})^2}, \frac{\eta_2}{3 (L_V^{(\sigma)})^2}, \frac{1}{b L_\sigma} \Big\}$, we have
	\begin{equation*}
		\sum_{t=1}^T \left( \left( \frac{1}{m} \sum_{i=1}^m \lVert \bg^{\bU}_{i,t} \rVert_F^2 \right) + \lVert \bg^{\bV}_t \rVert_F^2 + \left( \bg^\sigma_t \right)^2 \right)
		\leq \frac{3 \Delta_T}{ \min \{ \eta_1, \eta_2, \eta_3 \} }
	\end{equation*}
	where
	\begin{align*}
		\bg^{\bV}_t =&\; P_{\mathcal{T}_{\bV_{t-1}}} \left( \nabla_{\bV_{t-1}} f^{\mathrm{pca}}(\{\bU_{i,t}\}_i, \bV_{t-1}, \sigma_{t-1}) \right), \\
		\bg^{\bU}_{i,t} =&\; P_{\mathcal{T}_{\bU_{i,t-1}}} \left( \nabla_{\bU_{i,t-1}} f_i^{\mathrm{pca}}(\bU_{i,t-1}, \bV_{t-1},, \sigma_{t-1}) \right), \\
		\bg^\sigma_t =&\; \pt{\sigma_{t-1}} f^{\mathrm{pca}}(\{\bU_{i,t-1}\}_i, \bV_{t-1}, \sigma_{t-1}), \\
		\Delta_T =&\; f^{\mathrm{pca}}(\{\bU_{i,0}\}_i, \bV_0, \sigma_0) - f^{\mathrm{pca}}(\{\bU_{i,T}\}_i, \bV_{T}, \sigma_{T}).
	\end{align*}
\end{theorem}

\subsection{Proofs}

\begin{fact}\label{fact:grad}
The gradients of the local loss function with respect to the local and global PC's and $\sigma$ are given as
\begin{align*}
\nabla_{\bU_i} f_i^{\mathrm{pca}}(\bU_i, \bV, \sigma) &= {-}\frac{n}{2}(\bW_i^{-1} \bS_i\bW_i^{-1}\bU_i{-}\bW_i^{-1}\bU_i){+}\frac{\mathcal{P}_{\mathcal{T}_{\bV}}(\bU_{i})}{\sigma^2}, \\ 
\nabla_{\bV} f_i^{\mathrm{pca}}(\bU_i, \bV) &= -\frac{\mathcal{P}_{\mathcal{T}_{\bV}}(\bU_{i})(\bU_{i}^\top \bV+\bV^\top\bU_{i})}{2\sigma^2}, \\
\pt{\sigma} f_i^{\mathrm{pca}}(\bU_i, \bV, \sigma) &= \frac{d}{\sigma} - \frac{2 \xi + d^2(\bV, \bU_i)}{\sigma^3}, \\
\nabla_{\bU_i} \left( \pt{\sigma} f_i^{\mathrm{pca}}(\bU_i, \bV, \sigma) \right) &= - \frac{2 \mathcal{P}_{\mathcal{T}_{\bV}}(\bU_{i})}{\sigma^3}, \\
\nabla_{\bV} \left( \pt{\sigma} f_i^{\mathrm{pca}}(\bU_i, \bV, \sigma) \right) &= \frac{\mathcal{P}_{\mathcal{T}_{\bV}}(\bU_{i})(\bU_{i}^\top \bV+\bV^\top\bU_{i})}{\sigma^3}.
\end{align*}
\end{fact}
\begin{fact}\label{fact:norm}
For two matrices $\boldsymbol{A}\in \mathbb{R}^{a\times b}$ and $\boldsymbol{B}\in \mathbb{R}^{b\times c}$, we have
\begin{align*}
\|\boldsymbol{A} \boldsymbol{B}\|_F \leq \|\boldsymbol{A}\|_{op}\| \boldsymbol{B}\|_F \text{ and } \|\boldsymbol{A} \boldsymbol{B}\|_F \leq \|\boldsymbol{A}\|_{F}\| \boldsymbol{B}\|_{op} .
\end{align*}
\end{fact}
\begin{fact}\label{fact:lips multip}
For matrix to matrix functions, $\{g_i\}_{i=1}^k$, with bounded output operator norms, $\max_{\bX} \|g_i(\bX)\|_{op}\leq M_i$, we have
\begin{align*}
 \lVert \prod_{i=1}^k g_i(\bX)-\prod_{i=1}^k g_i(\bY)\rVert_F \leq \prod_{j=1}^k M_j \left(\sum_{i=1}^k \|g_i(\bX)-g_i(\bY)\rVert_F\right)
\end{align*}
\end{fact}

\subsubsection{Proof of Lemma~\ref{lem:sigma_lb}} \label{app:lem sigma lb}
\begin{proof}
We use mathematical induction to proof the lemma. For the base case, it is given that $\sigma_0 \geq \omega \sqrt{\frac{2 \xi}{d}}$. Assume that for all $\tau \in \{0, 1, \dots, t\}$,
\begin{equation*}
\sigma_\tau \geq \omega \sqrt{\frac{2 \xi}{d}}.
\end{equation*}
Then, we consider the following two cases. First, if $\sigma_t \in \left[ \omega \sqrt{\frac{2 \xi}{d}}, \sqrt{\frac{2 \xi}{d}} \right]$, we have
\begin{align*}
& \sigma_t \leq \sqrt{\frac{2 \xi}{d}}
\quad\Rightarrow\quad d \leq \frac{2 \xi}{\sigma_t^2}
\quad\Rightarrow\quad \frac{d}{\sigma_t} - \frac{2 \xi}{\sigma_t^3} \leq 0 \\
\Rightarrow\quad& \forall i \in [m]: \quad \pst f_i^{\mathrm{pca}}(\bU_{i,t}, \bV_t ,\sigma_t) = \frac{d}{\sigma_t} - \frac{2 \xi + d^2(\bV_t, \bU_{i, t})}{\sigma_t^3} \leq \frac{d}{\sigma_t} - \frac{2 \xi}{\sigma_t^3} \leq 0 \\
\Rightarrow\quad& \forall i \in [m]: \quad \sigma_{i, t+1} = \sigma_t - \eta_3 \pst f_i^{\mathrm{pca}}(\bU_{i,t}, \bV_t, \sigma_t) \geq \sigma_t \geq \omega \sqrt{\frac{2 \xi}{d}} \\
\Rightarrow\quad& \sigma_{t+1} = \frac{1}{m} \sum_{i=1}^m  \sigma_{i,t+1} \geq \omega \sqrt{\frac{2 \xi}{d}}.
\end{align*}
Otherwise, if we have $\sigma_t > \sqrt{\frac{2 \xi}{d}}$, we have
\begin{align*}
\sigma_{i, t+1}
=&\; \sigma_t - \eta_3 \pst f_i^{\mathrm{pca}}(\bU_{i,t}, \bV_t, \sigma_t)
= \sigma_t - \eta_3 \left( \frac{d}{\sigma_t} - \frac{2 \xi + d^2(\bV_t, \bU_{i, t})}{\sigma_t^3} \right) \\
\geq&\; \sigma_t - \eta_3 \frac{d}{\sigma_t}
\geq \sqrt{\frac{2 \xi}{d}} - (1-\omega) \frac{2 \xi}{d^2} \cdot \frac{d}{\sqrt{2 \xi/d}}
= \omega \sqrt{\frac{2 \xi}{d}}
\end{align*}
and thus
\begin{align*}
\sigma_{t+1}
=&\; \frac{1}{m} \sum_{i=1}^m \sigma_{i,t+1}
\geq \omega \sqrt{\frac{2 \xi}{d}}.
\end{align*}
Thus, by mathematical induction, we have
\begin{equation*}
\forall t \in \mathbb{N} \;\; \forall i \in [m]: \quad \sigma_t \geq \omega \sqrt{\frac{2 \xi}{d}} \quad \text{and} \quad \sigma_{i,t} \geq \omega \sqrt{\frac{2 \xi}{d}}.
\end{equation*}
\end{proof}

\subsubsection{Proof of Lemma~\ref{lem:sigma}}
\begin{proof}
For the Lipschitz smoothness, we have
\begin{align*}
\left\vert \frac{\partial^2}{\partial \sigma^2} f_i^{\mathrm{pca}}(\bU_i, \bV, \sigma) \right\vert
=&\; \left\vert \frac{6 \xi + 3 d^2(\bV, \bU_i)}{\sigma^4} - \frac{d}{\sigma^2} \right\vert \\
\leq&\; \left\vert \frac{6 \xi + 3 d^2(\bV, \bU_i)}{\sigma^4} \right\vert + \left\vert \frac{d}{\sigma^2} \right\vert \\
\leq&\; \frac{6 \xi + 12}{4 \xi^2 \omega^4 / d^2} + \frac{d}{2 \xi \omega^2 / d} \\
=&\; \frac{3 d^2}{2 \xi \omega^4} + \frac{3 d^2}{\xi^2 \omega^4} + \frac{d^2}{2 \xi \omega^2} \\
=&\; L_\sigma
\end{align*}
for any $\bV$, $\bU_i$, and $\sigma \geq \omega \sqrt{\frac{2 \xi}{d}}$.
%For the bound on the gradients, we have
%\begin{align*}
%\left\vert \pt{\sigma} f_i^{\mathrm{pca}}(\bU_i, \bV, \sigma) \right\vert
%=&\; \left\vert \frac{d}{\sigma} - \frac{2 \xi + d^2(\bU_i, \bV)}{\sigma^3} \right\vert \\
%\leq&\; \left\vert \frac{d}{\sigma} \right\vert + \left\vert \frac{2 \xi + d^2(\bU_i, \bV)}{\sigma^3} \right\vert \\
%\leq&\; \frac{d}{\omega \sqrt{2 \xi / d}} + \frac{2 \xi + D^2}{\omega^3(2 \xi / d)^{\frac{3}{2}}} \\
%=&\; \frac{d^\frac{3}{2}}{\omega \sqrt{2 \xi}} + \frac{d^\frac{3}{2}}{\omega^3 \sqrt{2 \xi}} + \frac{d^\frac{3}{2} D^2}{\omega^3 (2 \xi)^\frac{3}{2}} \\
%=&\; G_\sigma
%\end{align*}
%for any $\bV$, $\bU_i$, and $\sigma \geq \omega \sqrt{\frac{2 \xi}{d}}$.
\end{proof}

\subsubsection{Proof of Lemma~\ref{lem:U}}
\begin{proof}
For the bound on the gradient,
\begin{align*}
&\left\Vert-\frac{n}{2}(\bW_i^{-1}S_i\bW_i^{-1}\bU_i-\bW_i^{-1}\bU_i){+}\frac{\mathcal{P}_{\mathcal{T}_{\bV}}(\bU_{i})}{\sigma^2}\right\Vert_{op} \\
\leq& \|-\frac{n}{2}(\bW_i^{-1}S_i\bW_i^{-1}\bU_i-\bW_i^{-1}\bU_i)\|_{op} + \frac{2}{\sigma^2} \\
\leq& \frac{n}{2}(\|\bW_i^{-1}S_i\bW_i^{-1}\bU_i\|_{op} + \|\bW_i^{-1}\bU_i\|_{op}) + \frac{2}{\sigma^2} \\
\leq& \frac{n}{2}\left(\frac{G_{max,op}}{\sigma_{\epsilon}^4}+\frac{1}{\sigma_{\epsilon}^2}\right) + \frac{d}{\xi \omega^2},
\end{align*} 
where in the last inequality we use $\|\bW_i^{-1}\|_{op} \leq \frac{1}{\sigma_{\epsilon}^2}$. Therefore, we find that norm of the gradient is bounded by $G_U:= \frac{n}{2}(\frac{G_{max,op}}{\sigma_{\epsilon}^4}+\frac{1}{\sigma_{\epsilon}^2}) + \frac{d}{\xi \omega^2}$. For the Lipschitz continuity of the gradient, we omit the client index $i$ and use $\bU_1$ and $\bU_2$ to denote two arbitrary points on $St(d,r)$ for simplicity. For any client $i$, we focus on the first term of the gradient,
\begin{align}\label{eq:lip U p1}
&\|\bW_1^{-1}S_i\bW_1^{-1}\bU_1 {-}\bW_1^{-1}\bU_1{-} \bW_2^{-1}S_i\bW_2^{-1}\bU_2{+}\bW_2^{-1}\bU_2\|_F \notag \\
& \leq  \Big(\frac{1}{\sigma_{\epsilon}^2}+\frac{G_{max,op}}{\sigma_{\epsilon}^4} +\Big(1 + \frac{2G_{max,op}}{\sigma_{\epsilon}^2}\Big) \frac{2}{\sigma_{\epsilon}^4}\Big)\|\bU_2-\bU_1\|_F,
\end{align}
For the second part of the gradient we have
\begin{align*}
&\frac{1}{\sigma^2}\|\mathcal{P}_{\mathcal{T}_{\bV}}(\bU_{1})-\mathcal{P}_{\mathcal{T}_{\bV}}(\bU_{2})\|_F \\
&=\frac{1}{\sigma^2}\|\bU_1{-}\bU_2{-}\frac{1}{2}\bV(\bV^\top(\bU_1{-}\bU_2){+}(\bU_1^\top{-}\bU_2^\top)\bV)\|_F\\
&\leq \frac{2}{\sigma^2}\|\bU_1-\bU_2\|_F, \\
&\leq \frac{d}{\xi \omega^2} \|\bU_1-\bU_2\|_F,
\end{align*}
where in the last inequality we use Fact~\ref{fact:norm}. As a result, we find that the gradient is Lipschitz continuous with $L_U:=\frac{n}{2}\Big(\frac{1}{\sigma_{\epsilon}^2}+\frac{G_{max,op}}{\sigma_{\epsilon}^4} +\Big(1 + \frac{2G_{max,op}}{\sigma_{\epsilon}^2}\Big) \frac{2}{\sigma_{\epsilon}^4}\Big) + \frac{d}{\xi \omega^2}$.
\end{proof}

\subsubsection{Proof of Lemma~\ref{lem:V}}
\begin{proof}
For the Lipschitz constant
\begin{align*}
&\frac{2}{\sigma^2}\|\mathcal{P}_{\mathcal{T}_{\boldsymbol{V}_1}}(\bU_{i})\text{sym}(\bU_i^\top\bV_1){-} \mathcal{P}_{\mathcal{T}_{\boldsymbol{V_2}}}(\bU_{i})\text{sym}(\bU_i^\top\bV_2)\|_F \\
=&\; \frac{2}{\sigma^2} \|\bU_i\bU_i^\top(\bV_1-\bV_2)+\bU_i(\bV_1-\bV_2)\bU_i^\top
- \frac{1}{2}(\bV_1(\bV_1^\top\bU_i+\bU_i^\top\bV_1)-\bV_2(\bV_2^\top\bU_i+\bU_i^\top\bV_2))\|_F\\
\leq&\; \frac{24}{\sigma^2}\|\bV_1-\bV_2\|_F \\
\leq&\; \frac{12 d}{\xi \omega^2},
\end{align*}
where $\text{sym}(\bU_i^\top\bV)=\bU_{i}^\top \bV{+}\bV^\top\bU_{i}$ and we used Fact~\ref{fact:lips multip}, hence $L_V = \frac{12 d}{\xi \omega^2}$. For the gradient bound, is it straightforward to see that
\begin{align*}
\lVert \nabla_{\bV} f_i^{\mathrm{pca}}(\bU, \bV, \sigma) \rVert_2
\leq \frac{4}{\sigma^2}
\leq \frac{2 d}{\xi \omega^2}.
\end{align*}
\end{proof}

\subsubsection{Proof of Lemma~\ref{lem:UV_sigma_Lip}}
\begin{proof}
Using Fact~\ref{fact:grad}, we have
\begin{align*}
\lVert \nabla_{\bU_i} \left( \pt{\sigma} f_i^{\mathrm{pca}}(\bU_i, \bV, \sigma) \right) \rVert_2
\leq \frac{4}{\sigma^3}
\leq \frac{\sqrt{2 d^3}}{\omega^3 \sqrt{\xi^3}}
\end{align*}
and
\begin{align*}
\lVert \nabla_{\bV} \left( \pt{\sigma} f_i^{\mathrm{pca}}(\bU_i, \bV, \sigma) \right) \rVert_2
\leq \frac{8}{\sigma^3}
\leq \frac{2 \sqrt{2 d^3}}{\omega^3 \sqrt{\xi^3}}
\end{align*}
\end{proof}

\subsubsection{Proof of Lemma~\ref{lem:suff_dec}}
\begin{proof}
We have
\begin{align*}
&f^{\mathrm{pca}}(\{\bU_{i,t}\}_i, \bV_t, \sigma_t) - f^{\mathrm{pca}}(\{\bU_{i,t-1}\}_i, \bV_{t-1}, \sigma_{t-1}) \\
&= \left[ f^{\mathrm{pca}}(\{\bU_{i,t}\}_i, \bV_t, \sigma_t) - f^{\mathrm{pca}}(\{\bU_{i,t}\}_i, \bV_{t}, \sigma_{t-1}) \right] \\
&+ \left[ f^{\mathrm{pca}}(\{\bU_{i,t}\}_i, \bV_{t}, \sigma_{t-1}) - f^{\mathrm{pca}}(\{\bU_{i,t}\}_i, \bV_{t-1}, \sigma_{t-1}) \right] \\
&+ \left[ f^{\mathrm{pca}}(\{\bU_{i,t}\}_i, \bV_{t-1}, \sigma_{t-1}) - f^{\mathrm{pca}}(\{\bU_{i,t-1}\}_i, \bV_{t-1}, \sigma_{t-1}) \right]
\end{align*}
With a similar proof as Lemma~5 in~\cite{ozkara2023isit}, we have
\begin{align*}
f^{\mathrm{pca}}(\{\bU_{i,t}\}_i, \bV_{t-1}, \sigma_{t-1})-&f^{\mathrm{pca}}(\{\bU_{i,t-1}\}_i, \bV_{t-1}, \sigma_{t-1}) \\
&\leq \frac{({-}\eta_1 {+} C_{\eta_1} \eta_1^2)}{m} \sum_{i=1}^m \| P_{\mathcal{T}_{\bU_{i,t{-}1}}}(\nabla_{\bU_{i,t{-}1}} f_i^{\mathrm{pca}}(\bU_{i,t{-}1},\bV_{t{-}1}, \sigma_{i,t}))\|_F^2, \\
f^{\mathrm{pca}}(\{\bU_{i,t}\}_i, \bV_t, \sigma_{t-1}) - &f^{\mathrm{pca}}(\{\bU_{i,t}\}_i, \bV_{t-1}, \sigma_{t-1}) \\
&\leq ( -\eta_2 + C{\eta_2} \eta_2^2 ) \lVert P_{\mathcal{T}_{\bV_{t-1}}}(\nabla_{\bV_{t-1}} f^{\mathrm{pca}}(\{\bU_{i,t}\}_i, \bV_{t-1}, \sigma_{t-1}))\rVert_F^2
\end{align*}
where
\begin{gather*}
C_{\eta_1} = C_1G_1+\frac{L_{gu} ( C_1^2G_1^2 + 1) }{2}, \\
C{\eta_2} = C_2G_2+\frac{L_{gv}  ( C_2^2G_2^2 + 1) }{2}, \\
G_1 = 2 G_U \sqrt{d}, \\
G_2 = 2 G_V \sqrt{d}
\end{gather*}
with some constants $C_1, C_2$ given by Lemma~\ref{lem:retraction} and $G_U, G_V$ given in Lemma~\ref{lem:U},~\ref{lem:V}.
For the sufficient decrease with respect to $\sigma$, we have
\begin{align*}
&\; f^{\mathrm{pca}}(\{\bU_{i,t}\}_i, \bV_t, \sigma_t) - f^{\mathrm{pca}}(\{\bU_{i,t}\}_i, \bV_t, \sigma_{t-1}) \\
\leq&\; \pt{\sigma_{t-1}} f^{\mathrm{pca}}(\{\bU_{i,t}\}_i, \bV_t, \sigma_{t-1}) (\sigma_t - \sigma_{t-1}) + \frac{L_\sigma}{2} (\sigma_t - \sigma_{t-1})^2 \\
=&\; \left[ \pt{\sigma_{t-1}} f^{\mathrm{pca}}(\{\bU_{i,t}\}_i, \bV_t, \sigma_{t-1}) \right] \left[ - \eta_3 \pt{\sigma_{t-1}} f^{\mathrm{pca}}(\{\bU_{i,t-1}\}_i, \bV_{t-1}, \sigma_{t-1}) \right] + \frac{\eta_3^2 L_\sigma}{2} \left[ \pt{\sigma_{t-1}} f^{\mathrm{pca}}(\{\bU_{i,t-1}\}_i, \bV_{t-1}, \sigma_{t-1}) \right]^2 \tag{by the update rule of $\sigma_t$} \\
=&\; (-\eta_3) \left[ \pt{\sigma_{t-1}} f^{\mathrm{pca}}(\{\bU_{i,t}\}_i, \bV_t, \sigma_{t-1}) -\pt{\sigma_{t-1}} f^{\mathrm{pca}}(\{\bU_{i,t-1}\}_i, \bV_{t-1}, \sigma_{t-1}) + \pt{\sigma_{t-1}} f^{\mathrm{pca}}(\{\bU_{i,t-1}\}_i, \bV_{t-1}, \sigma_{t-1}) \right] \\
&\quad \cdot \left[ \pt{\sigma_{t-1}} f^{\mathrm{pca}}(\{\bU_{i,t-1}\}_i, \bV_{t-1}, \sigma_{t-1}) \right] + \frac{\eta_3^2 L_\sigma}{2} \left[ \pt{\sigma_{t-1}} f^{\mathrm{pca}}(\{\bU_{i,t-1}\}_i, \bV_{t-1}, \sigma_{t-1}) \right]^2 \\
=&\; (-\eta_3) \left[ \pt{\sigma_{t-1}} f^{\mathrm{pca}}(\{\bU_{i,t}\}_i, \bV_t, \sigma_{t-1}) -\pt{\sigma_{t-1}} f^{\mathrm{pca}}(\{\bU_{i,t-1}\}_i, \bV_{t-1}, \sigma_{t-1}) \right] \left[ \pt{\sigma_{t-1}} f^{\mathrm{pca}}(\{\bU_{i,t-1}\}_i, \bV_{t-1}, \sigma_{t-1}) \right] \\
& \quad - \left( \eta_3 - \frac{\eta_3^2 L_\sigma}{2} \right) \left[ \pt{\sigma_{t-1}} f^{\mathrm{pca}}(\{\bU_{i,t-1}\}_i, \bV_{t-1}, \sigma_{t-1}) \right]^2 \\
\leq& \frac{\eta_3}{2} \left[ \pt{\sigma_{t-1}} f^{\mathrm{pca}}(\{\bU_{i,t}\}_i, \bV_t, \sigma_{t-1}) -\pt{\sigma_{t-1}} f^{\mathrm{pca}}(\{\bU_{i,t-1}\}_i, \bV_{t-1}, \sigma_{t-1}) \right]^2 + \frac{\eta_3}{2} \left[ \pt{\sigma_{t-1}} f^{\mathrm{pca}}(\{\bU_{i,t-1}\}_i, \bV_{t-1}, \sigma_{t-1}) \right]^2 \\
&\quad - \left( \eta_3 - \frac{\eta_3^2 L_\sigma}{2} \right) \left[ \pt{\sigma_{t-1}} f^{\mathrm{pca}}(\{\bU_{i,t-1}\}_i, \bV_{t-1}, \sigma_{t-1}) \right]^2 \tag{since $2ab \leq a^2 + b^2$ for any $a, b \in \mathbb{R}$} \\
=&\; \frac{\eta_3}{2} \left[ \pt{\sigma_{t-1}} f^{\mathrm{pca}}(\{\bU_{i,t}\}_i, \bV_t, \sigma_{t-1}) -\pt{\sigma_{t-1}} f^{\mathrm{pca}}(\{\bU_{i,t-1}\}_i, \bV_{t-1}, \sigma_{t-1}) \right]^2 -  \left( \frac{\eta_3}{2} - \frac{\eta_3^2 L_\sigma}{2} \right) \left[ \pt{\sigma_{t-1}} f^{\mathrm{pca}}(\{\bU_{i,t-1}\}_i, \bV_{t-1}, \sigma_{t-1}) \right]^2 \\
=&\; \frac{\eta_3}{2} \Bigg[ \pt{\sigma_{t-1}} f^{\mathrm{pca}}(\{\bU_{i,t}\}_i, \bV_t, \sigma_{t-1}) -\pt{\sigma_{t-1}} f_i^{\mathrm{pca}}(\bV_{t}, \{ \bU_{i,t-1} \}_i, \sigma_{t-1}) \\
&\quad\quad + \pt{\sigma_{t-1}} f^{\mathrm{pca}}(\{\bU_{i,t-1}\}_i, \bV_t, \sigma_{t-1}) -\pt{\sigma_{t-1}} f^{\mathrm{pca}}(\{\bU_{i,t-1}\}_i, \bV_{t-1}, \sigma_{t-1}) \Bigg]^2 \\
&\quad\quad - \left( \frac{\eta_3}{2} - \frac{\eta_3^2 L_\sigma}{2} \right) \left[ \pt{\sigma_{t-1}} f^{\mathrm{pca}}(\{\bU_{i,t-1}\}_i, \bV_{t-1}, \sigma_{t-1}) \right]^2 \\
=&\; \frac{\eta_3}{2} \Bigg[ \left( \frac{1}{m} \sum_{i=1}^m \pt{\sigma_{t-1}} f_i^{\mathrm{pca}}(\bU_{i,t}, \bV_t, \sigma_{t-1}) -\pt{\sigma_{t-1}} f_i^{\mathrm{pca}}(\bU_{i,t-1}, \bV_{t}, \sigma_{t-1}) \right) \\
&\quad\quad + \pt{\sigma_{t-1}} f^{\mathrm{pca}}(\{\bU_{i,t-1}\}_i, \bV_t, \sigma_{t-1}) -\pt{\sigma_{t-1}} f^{\mathrm{pca}}(\{\bU_{i,t-1}\}_i, \bV_{t-1}, \sigma_{t-1}) \Bigg]^2 \\
&\quad\quad - \left( \frac{\eta_3}{2} - \frac{\eta_3^2 L_\sigma}{2} \right) \left[ \pt{\sigma_{t-1}} f^{\mathrm{pca}}(\{\bU_{i,t-1}\}_i, \bV_{t-1}, \sigma_{t-1}) \right]^2 \\
\leq&\; \frac{\eta_3}{2} \left[ \left( \frac{1}{m} \sum_{i=1}^m L_U^{(\sigma)} \lVert \bU_{i,t} - \bU_{i,t-1} \rVert_F \right) + L_V^{(\sigma)} \lVert \bV_{i,t} - \bV_{i,t-1} \rVert_F \right]^2  -  \left( \frac{\eta_3 - \eta_3^2 L_\sigma}{2} \right) \left[ \pt{\sigma_{t-1}} f^{\mathrm{pca}}(\{\bU_{i,t-1}\}_i, \bV_{t-1}, \sigma_{t-1}) \right]^2 \tag{from Lemma~\ref{lem:UV_sigma_Lip}} \\
\leq&\; \frac{\eta_3}{2} \left[ 2 \left( \frac{1}{m} \sum_{i=1}^m L_U^{(\sigma)} \lVert \bU_{i,t} - \bU_{i,t-1} \rVert_F \right)^2 + 2 \left( L_V^{(\sigma)} \lVert \bV_{i,t} - \bV_{i,t-1} \rVert_F \right)^2 \right]  -  \left( \frac{\eta_3 - \eta_3^2 L_\sigma}{2} \right) \left[ \pt{\sigma_{t-1}} f^{\mathrm{pca}}(\{\bU_{i,t-1}\}_i, \bV_{t-1}, \sigma_{t-1}) \right]^2 \tag{since $(a+b)^2 \leq 2 a^2 + 2 b^2$} \\
\leq&\; \eta_3 \left[ \frac{1}{m} \sum_{i=1}^m \left( L_U^{(\sigma)} \lVert \bU_{i,t} - \bU_{i,t-1} \rVert_F \right)^2 + \left( L_V^{(\sigma)} \lVert \bV_{i,t} - \bV_{i,t-1} \rVert_F \right)^2 \right]  -  \left( \frac{\eta_3 - \eta_3^2 L_\sigma}{2} \right) \left[ \pt{\sigma_{t-1}} f^{\mathrm{pca}}(\{\bU_{i,t-1}\}_i, \bV_{t-1}, \sigma_{t-1}) \right]^2 \tag{Cauchy–Schwarz inequality} \\
=&\; \eta_3 (L_U^{(\sigma)})^2 \frac{1}{m} \left( \sum_{i=1}^m \lVert P_{\mathcal{T}_{\bU_{i,t-1}}} \left( \nabla_{\bU_{i,t-1}} f^{\mathrm{pca}}(\bU_{i,t-1}, \bV_{t-1}, \sigma_{t-1}) \right) \rVert_F^2 \right) + \eta_3 (L_V^{(\sigma)})^2 \lVert P_{\mathcal{T}_{\bV_{t-1}}} \left( \nabla_{\bV_{t-1}} f^{\mathrm{pca}}(\{\bU_{i,t}\}_i, \bV_{t-1}, \sigma_{t-1}) \right) \rVert_F^2 \\
&\quad - \left( \frac{\eta_3 - \eta_3^2 L_\sigma}{2} \right) \left[ \pt{\sigma_{t-1}} f^{\mathrm{pca}}(\{\bU_{i,t-1}\}_i, \bV_{t-1}, \sigma_{t-1}) \right]^2.
\end{align*}
Thus, we have
\begin{align*}
&f^{\mathrm{pca}}(\{\bU_{i,t}\}_i, \bV_t, \sigma_t) - f^{\mathrm{pca}}(\{\bU_{i,t-1}\}_i, \bV_{t-1}, \sigma_{t-1}) \\
=&\; \left[ f^{\mathrm{pca}}(\{\bU_{i,t}\}_i, \bV_{t-1}, \sigma_{t-1}) - f^{\mathrm{pca}}(\{\bU_{i,t-1}\}_i, \bV_{t-1}, \sigma_{t-1}) \right] \\
& + \left[ f^{\mathrm{pca}}(\{\bU_{i,t}\}_i, \bV_t, \sigma_{t-1}) - f^{\mathrm{pca}}(\{\bU_{i,t}\}_i, \bV_{t-1}, \sigma_{t-1}) \right] \\
& + \left[ f^{\mathrm{pca}}(\{\bU_{i,t}\}_i, \bV_t, \sigma_t) - f^{\mathrm{pca}}(\{\bU_{i,t}\}_i, \bV_t, \sigma_{t-1}) \right] \\
\leq&\; \left( -\eta_1 + C_{\eta_1} \eta_1^2 + \eta_3 (L_U^{(\sigma)})^2 \right) \frac{1}{m} \sum_{i=1}^m \lVert P_{\mathcal{T}_{\bU_{i,t-1}}} \left( \nabla_{\bU_{i,t-1}} f_i^{\mathrm{pca}}(\bU_{i,t-1}, \bV_{t-1},, \sigma_{t-1}) \right) \rVert_F^2 \\
& + \left( -\eta_2 + C{\eta_2} \eta_2^2 + \eta_3 (L_V^{(\sigma)})^2 \right) \lVert P_{\mathcal{T}_{\bV_{t-1}}} \left( \nabla_{\bV_{t-1}} f^{\mathrm{pca}}(\{\bU_{i,t}\}_i, \bV_{t-1}, \sigma_{t-1}) \right) \rVert_F^2 \\
& + \left( \frac{- \eta_3 + \eta_3^2 L_\sigma}{2} \right) \left[ \pt{\sigma_{t-1}} f^{\mathrm{pca}}(\{\bU_{i,t-1}\}_i, \bV_{t-1}, \sigma_{t-1}) \right]^2.
\end{align*}
By choosing $\eta_1 = \min\{ \frac{1}{3C_{\eta_1}}, 1\}$, $\eta_2 = \min\{ \frac{1}{3C{\eta_2}}, 1\}$, and $\eta_3 = \min \Big\{ \frac{\eta_1}{3 (L_U^{(\sigma)})^2}, \frac{\eta_2}{3 (L_V^{(\sigma)})^2}, \frac{1}{6 L_\sigma} \Big\}$, we have
\begin{gather*}
-\eta_1 + C_{\eta_1} \eta_1^2 + \eta_3 (L_U^{(\sigma)})^2
\leq \eta_1 \left( C_{\eta_1} \eta_1 - \frac{1}{3} \right) - \frac{2 \eta_1}{3} + \frac{\eta_1}{3} = - \frac{\eta_1}{3}, \\
-\eta_2 + C{\eta_2} \eta_2^2 + \eta_3 (L_V^{(\sigma)})^2
\leq\eta_2 \left( C{\eta_2} \eta_2 - \frac{1}{3} \right) - \frac{2 \eta_2}{3} + \frac{\eta_2}{3} = - \frac{\eta_2}{3}, \\
\frac{- \eta_3 + \eta_3^2 L_\sigma}{2}
= \eta_3 \left( L_\sigma \eta_3 - \frac{1}{6} \right) - \frac{\eta_3}{3}
\leq -\frac{\eta_3}{3}.
\end{gather*}
Therefore, we obtain
\begin{align*}
f^{\mathrm{pca}}(\{\bU_{i,t}\}_i, \bV_t, \sigma_t) - &f^{\mathrm{pca}}(\{\bU_{i,t-1}\}_i, \bV_{t-1}, \sigma_{t-1})
\leq- \frac{\eta_1}{3} \left( \frac{1}{m}  \sum_{i=1}^m \lVert P_{\mathcal{T}_{\bU_{i,t-1}}} \left( \nabla_{\bU_{i,t-1}} f_i^{\mathrm{pca}}(\bU_{i,t-1}, \bV_{t-1},, \sigma_{t-1}) \right) \rVert_F^2 \right) \\
& \quad - \frac{\eta_2}{3} \lVert P_{\mathcal{T}_{\bV_{t-1}}} \left( \nabla_{\bV_{t-1}} f^{\mathrm{pca}}(\{\bU_{i,t}\}_i, \bV_{t-1}, \sigma_{t-1}) \right) \rVert_F^2
- \frac{\eta_3}{3} \left[ \pt{\sigma_{t-1}} f^{\mathrm{pca}}(\{\bU_{i,t-1}\}_i, \bV_{t-1}, \sigma_{t-1}) \right]^2.
\end{align*}
\end{proof}

\subsubsection{Proof of Theorem~\ref{thm:PCA_convergence}}
\begin{proof}
Following Lemma~\ref{lem:suff_dec}, by telescoping across the iterations, we have
\begin{align*}
& \frac{1}{T} \sum_{t=1}^T \Bigg[ \lVert P_{\mathcal{T}_{\bV_{t-1}}} \left( \nabla_{\bV_{t-1}} f^{\mathrm{pca}}(\{\bU_{i,t}\}_i, \bV_{t-1}, \sigma_{t-1}) \right) \rVert_F^2 \\
&+ \left( \frac{1}{m} \sum_{i=1}^m \lVert P_{\mathcal{T}_{\bU_{i,t-1}}} \left( \nabla_{\bU_{i,t-1}} f_i^{\mathrm{pca}}(\bU_{i,t-1}, \bV_{t-1},, \sigma_{t-1}) \right) \rVert_F^2 \right)
+ \left[ \pt{\sigma_{t-1}} f^{\mathrm{pca}}(\{\bU_{i,t-1}\}_i, \bV_{t-1}, \sigma_{t-1}) \right]^2 \Bigg] \\
\leq&\; \frac{1}{T \min \{ \frac{\eta_1}{3}, \frac{\eta_2}{3}, \frac{\eta_3}{3} \}} \sum_{t=1}^T f^{\mathrm{pca}}(\{\bU_{i,t-1}\}_i, \bV_{t-1}, \sigma_{t-1}) - f^{\mathrm{pca}}(\{\bU_{i,t}\}_i, \bV_{t},, \sigma_{t}) \\
=&\; \frac{3 \left( f^{\mathrm{pca}}(\{\bU_{i,0}\}_i, \bV_0, \sigma_0) - f^{\mathrm{pca}}(\{\bU_{i,T}\}_i, \bV_{T}, \sigma_{T}) \right)}{T \min \{ \eta_1, \eta_2, \eta_3 \}}.
\end{align*}
\end{proof}

\section{Proofs for Adaptive AEs: \texttt{ADEPT-AE}} \label{app:ae}

%		\STATE $\vct{\theta}_{i,t} = \vct{\theta}_{i,t-1} - \eta_{1} \nabla_{\vct{\theta}_{i,t-1}} f_i(\vct{\theta}_{i,t-1}, \vct{\mu}_{t-1}, \sigma_{t-1})$
%		\IF{$\tau$ divides $t$}
%		\STATE $\vct{\mu}_{i,t} = \vct{\mu}_{t-1} - \eta_{2} \nabla_{\vct{\mu}_{t-1}} f_i(\vct{\theta}_{i,t}, \vct{\mu}_{t-1}, \sigma_{t-1})$
%		\STATE $\sigma_{i,t} = \sigma_{t-1} - \eta_{3} \pt{\sigma_{t-1}} f_i(\vct{\theta}_{i, t}, \vct{\mu}_{t-1}, \sigma_{t-1})$
%		\STATE Send $\vct{\mu}_{i,t}, \sigma_{i,t}$ to server
%		\ELSE
%		\STATE $\vct{\mu}_{t} = \vct{\mu}_{t-1}$, $\sigma_{t} = \sigma_{t-1}$
%		\ENDIF
%		\ENDFOR
%		\STATE \textbf{At the Server:}
%		\IF{$\tau$ divides $t$}
%		\STATE Receive $\{ \vct{\mu}_{i,t} \}_{i=1}^m$ and $\{ \sigma_{i,t} \}_{i=1}^m$
%		\STATE $\vct{\mu}_{t} = \frac{1}{m} \sum_{i=1}^m \vct{\mu}_{i,t}$, $\sigma_{t} = \frac{1}{m} \sum_{i=1}^m \sigma_{i,t}$
%		\STATE Broadcast $\vct{\mu}_{t}$, $\sigma_{t}$ to all clients
%		\ENDIF
%		\ENDFOR
%	\end{algorithmic}
%	{\bf Output:} Personalized autoencoders $\{ \vct{\theta}_{1,T}, \ldots, \vct{\theta}_{m,T} \}$.
%\end{algorithm}

%(For vector sigma)
%\begin{align}\label{eq:opt_problem_ae_2}
%	\argmin_{\{\vct{\theta}_i\}, \vct{\mu} , \vct{\sigma}} \quad f^{\text{ae}}(\{\vct{\theta}_i\}, \vct{\mu} , \vct{\sigma})   = \frac{1}{m}\sum_{i=1}^m \left( \|\mat{X}_i - \psi_{\vct{\theta}_i^{\mathrm{d}}}( g_{\vct{\theta}_i^{\mathrm{e}}} (\mat{X}_i) ) \|^2_F  \notag
%	+ \sum_{j=1}^{d_\theta} \left( \frac{2 \xi + (\mu_j - \theta_{i,j} )^2}{2 \sigma_j^2}
%  	+ \log \sigma_j \right) \right).
%\end{align}

\subsection{Proofs}

\subsubsection{Proof of Lemma~\ref{lem:ae_lipschitz}}
\begin{proof}
Following the same proof in Lemma~\ref{lem:sigma_lb}, we have the same lower bound on $\sigma_t$ if we initialized it in the same way.

The gradient w.r.t. $\vct{\mu}$ is
\begin{equation*}
\nabla_{\vct{\mu}} f_i^{\mathrm{ae}}(\vct{\theta}, \vct{\mu}, \sigma) = \frac{\vct{\mu} - \vct{\theta}}{2 \sigma^2}.
\end{equation*}
Thus, we have
\begin{align*}
\left\Vert \frac{\vct{\mu}_1 - \vct{\theta}}{2 \sigma^2} - \frac{\vct{\mu}_2 - \vct{\theta}}{2 \sigma^2} \right\Vert
= \left\Vert \frac{\vct{\mu}_1 - \vct{\mu}_2}{2 \sigma^2} \right\Vert
\leq \frac{d_{\theta}}{2 \xi \omega^2} \Vert \vct{\mu}_1 - \vct{\mu}_2 \Vert
\leq L_{\vct{\mu}} \Vert \vct{\mu}_1 - \vct{\mu}_2 \Vert.
\end{align*}

For $L_\sigma$, we have
\begin{align*}
\left\vert \frac{\partial^2}{\partial \sigma^2} f_i^{\mathrm{ae}}(\vct{\theta}, \vct{\mu}, \sigma) \right\vert
&= \left\vert \frac{6 \xi + 3 \Vert \vct{\mu} - \vct{\theta} \Vert^2}{\sigma^4} - \frac{d_{\theta}}{\sigma^2} \right\vert \\
&\leq \left\vert \frac{6 \xi + 3 \Vert \vct{\mu} - \vct{\theta} \Vert^2}{\sigma^4} \right\vert + \left\vert \frac{d_{\theta}}{\sigma^2} \right\vert \\
&\leq \frac{6 \xi + 3 (2B)^2}{4 \xi^2 \omega^4 / d_{\theta}^2} + \frac{d_{\theta}^2}{2 \xi \omega^2} \\
&= \frac{3 \xi d_{\theta}^2}{2 \xi^2 \omega^4} + \frac{3 d_{\theta}^2 B^2}{\xi^2 \omega^4} + \frac{d_{\theta}^2}{2 \xi \omega^2} \\
&= L_\sigma.
\end{align*}

For $L_\sigma^{(\vct{\mu})}$, we have
\begin{align*}
\left\Vert \nabla_{\vct{\mu}} f_i^{\mathrm{ae}}(\vct{\theta}, \vct{\mu}, \sigma_1) - \nabla_{\vct{\mu}} f_i^{\mathrm{ae}}(\vct{\theta}, \vct{\mu}, \sigma_2) \right\Vert
&= \left\Vert \frac{\vct{\mu} - \vct{\theta}}{2 \sigma_1^2} -  \frac{\vct{\mu} - \vct{\theta}}{2 \sigma_2^2} \right\Vert \\
&= \left\vert \frac{1}{\sigma_1^2} - \frac{1}{\sigma_2^2} \right\vert \frac{\Vert \vct{\mu} - \vct{\theta} \Vert}{2} \\
&= \left\vert \sigma_1 - \sigma_2 \right\vert \left\vert \frac{1}{\sigma_1^2 \sigma_2} + \frac{1}{\sigma_1 \sigma_2^2} \right\vert \frac{\Vert \vct{\mu} - \vct{\theta} \Vert}{2} \\
&\leq 2 \left( \omega \sqrt{ \frac{2 \xi}{d_{\theta}} } \right)^{-3} B \vert \sigma_1 - \sigma_2 \vert \\
&= \frac{B \sqrt{d_{\theta}^3}}{\omega^3 \sqrt{2 \xi^3}} \vert \sigma_1 - \sigma_2 \vert \\
&= L_\sigma^{(\vct{\mu})} \vert \sigma_1 - \sigma_2 \vert.
\end{align*}

\end{proof}

\subsubsection{Proof of Theorem~\ref{thm:AE_convergence}}
\begin{proof}
Since $\vct{\mu}_t$ and $\sigma_t$ are updated only when $\tau$ divides $t$, we consider the two cases separately.

\paragraph{When $\tau$ divides $t$}

First, for the sufficient decrease of $\vct{\theta}_{t,i}$, we have
\begin{align*}
	& f_i^{\mathrm{ae}}(\vct{\theta}_{i,t}, \vct{\mu}_{t-1}, \sigma_{t-1}) - f_i^{\mathrm{ae}}(\vct{\theta}_{i,t-1}, \vct{\mu}_{t-1}, \sigma_{t-1}) \\
	& \leq \left\langle \nabla_{\vct{\theta}_{i,t-1}} f_i^{\mathrm{ae}}(\vct{\theta}_{i,t-1}, \vct{\mu}_{t-1}, \sigma_{t-1}), \vct{\theta}_{i,t} - \vct{\theta}_{i,t-1} \right\rangle + \frac{L_{\vct{\theta}}}{2} \lVert \vct{\theta}_{i,t} - \vct{\theta}_{i,t-1} \rVert^2 \\
	&= \left\langle \nabla_{\vct{\theta}_{i,t-1}} f_i^{\mathrm{ae}}(\vct{\theta}_{i,t-1}, \vct{\mu}_{t-1}, \sigma_{t-1}), - \eta_{1} \nabla_{\vct{\theta}_{i,t-1}} f_i^{\mathrm{ae}}(\vct{\theta}_{i,t-1}, \vct{\mu}_{t-1}, \sigma_{t-1}) \right\rangle
	+ \frac{L_{\vct{\theta}}}{2} \left\lVert - \eta_{1} \nabla_{\vct{\theta}_{i,t-1}} f_i^{\mathrm{ae}}(\vct{\theta}_{i,t-1}, \vct{\mu}_{i,t-1}, \sigma_{t-1}) \right\rVert^2 \\
	&\leq \left( -\eta_1+ \eta_1^2\frac{L_{\vct{\theta}}}{2} \right)\left\lVert \nabla_{\vct{\theta}_{i,t-1}} f_i^{\mathrm{ae}}(\vct{\theta}_{i,t-1}, \vct{\mu}_{t-1}, \sigma_{t-1}) \right\rVert^2.
\end{align*}
Sum over the clients and we have
\begin{align}
f^{\mathrm{ae}}(\{\vct{\theta}_{i,t}\}, \vct{\mu}_{t-1}, \sigma_{t-1}) - f^{\mathrm{ae}}(\{\vct{\theta}_{i,t-1}\}, \vct{\mu}_{t-1}, \sigma_{t-1})
\leq \left( -\eta_1+ \eta_1^2\frac{L_{\vct{\theta}}}{2} \right) \left( \frac{1}{m} \sum_{i=1}^m \left\lVert \nabla_{\vct{\theta}_{i,t-1}} f_i^{\mathrm{ae}}(\vct{\theta}_{i,t-1}, \vct{\mu}_{t-1}, \sigma_{t-1}) \right\rVert^2 \right). \label{eq:ae_suff_dec_theta}
\end{align}

Second, for the sufficient decrease of $\sigma_t$, define
\begin{gather*}
	g^{\sigma}_t = \frac{1}{m} \sum_{i=1}^m \pt{\sigma_{t-1}} f_i^{\mathrm{ae}}(\vct{\theta}_{i,t}, \vct{\mu}_{t-1}, \sigma_{t-1}).
\end{gather*}
Thus, $\sigma_t = \sigma_{t-1} - \eta_{2} g^{\sigma}_t$ and
\begin{align*}
	\: f_i^{\mathrm{ae}}(\vct{\theta}_{i,t}, \vct{\mu}_{t-1}, \sigma_{t}) - f_i^{\mathrm{ae}}(\vct{\theta}_{i,t}, \vct{\mu}_{t-1}, \sigma_{t-1})
	&\leq\; \left( \pt{\sigma_{t-1}} f_i^{\mathrm{ae}}(\vct{\theta}_{i,t}, \vct{\mu}_{t-1}, \sigma_{t-1}) \right) \left( \sigma_t - \sigma_{t-1} \right) + \frac{L_\sigma}{2} \left( \sigma_t - \sigma_{t-1} \right)^2 \\
	&=\; \left( \pt{\sigma_{t-1}} f_i^{\mathrm{ae}}(\vct{\theta}_{i,t}, \vct{\mu}_{t-1}, \sigma_{t-1}) \right) \left( - \eta_{2} g^{\sigma}_t \right) + \frac{L_\sigma}{2} \left( - \eta_{2} g^{\sigma}_t \right)^2.
\end{align*}
Sum over the clients and we have
\begin{align}
f^{\mathrm{ae}}(\{\vct{\theta}_{i,t}\}, \vct{\mu}_{t-1}, \sigma_{t}) - f^{\mathrm{ae}}(\{\vct{\theta}_{i,t}\}, \vct{\mu}_{t-1}, \sigma_{t-1})
\leq\;  -\eta_2 (g^{\sigma}_t)^2 + \eta_2^2 \frac{L_\sigma}{2} (g^{\sigma}_t)^2
= \left( -\eta_2 + \eta_2^2 \frac{L_\sigma}{2} \right) (g^{\sigma}_t)^2.\label{eq:ae_suff_dec_sigma}
\end{align}

Then, for the sufficient decrease of $\vct{\mu}_t$, define
\begin{alignat*}{2}
	\bg^{\vct{\mu}}_{i,t} &= \nabla_{\vct{\mu}_{t-1}} f_i^{\mathrm{ae}}(\vct{\theta}_{i,t}, \vct{\mu}_{t-1}, \sigma_{t-1}), \qquad &&\bg_t^{\vct{\mu}} = \frac{1}{m} \sum_{i=1}^m \bg^{\vct{\mu}}_{i,t}, \\
	\tilde{\bg}^{\vct{\mu}}_{i,t} &= \nabla_{\vct{\mu}_{t-1}} f_i^{\mathrm{ae}}(\vct{\theta}_{i,t}, \vct{\mu}_{t-1}, \sigma_{t}), \qquad
	&&\tilde{\bg}^{\vct{\mu}}_t = \frac{1}{m} \sum_{i=1}^m \tilde \bg^{\vct{\mu}}_{i,t}.
\end{alignat*}
We have
\begin{align*}
	\; f^{\mathrm{ae}}(\{\vct{\theta}_{i,t}\}, \vct{\mu}_t, \sigma_t) - f^{\mathrm{ae}}(\{\vct{\theta}_{i,t}\}, \vct{\mu}_{t-1}, \sigma_t) 
	& \leq\; \left\langle \tilde \bg_t^{\vct{\mu}}, \vct{\mu}_t - \vct{\mu}_{t-1} \right\rangle
	+ \frac{L_{\vct{\mu}}}{2} \lVert \vct{\mu}_t - \vct{\mu}_{t-1} \rVert^2 \\
	&=\; - \eta_3 \left\langle \tilde \bg_t^{\vct{\mu}} - \bg_t^{\vct{\mu}} + \bg_t^{\vct{\mu}} , \bg_t^{\vct{\mu}} \right\rangle
	+ \frac{L_{\vct{\mu}} \eta_3^2}{2} \lVert  \bg_t^{\vct{\mu}} \rVert^2 \\
	& = \left( -\eta_3 +\frac{L_{\vct{\mu}} \eta_3^2}{2} \right) \|\bg_t^{\vct{\mu}}\|^2 + \eta_3 \left\langle \bg_t^{\vct{\mu}} - \tilde \bg_t^{\vct{\mu}} , \bg_t^{\vct{\mu}} \right\rangle \\
	& \leq \left( -\eta_3 +\frac{L_{\vct{\mu}} \eta_3^2}{2} \right) \|\bg_t^{\vct{\mu}}\|^2 + \frac{\eta_3}{2} \|\bg_t^{\vct{\mu}} - \tilde \bg_t^{\vct{\mu}}\|^2 + \frac{\eta_3}{2} \| \bg_t^{\vct{\mu}} \|^2 \\
	& \leq \left( -\frac{\eta_3}{2} +\frac{L_{\vct{\mu}} \eta_3^2}{2} \right) \|\bg_t^{\vct{\mu}}\|^2 +  \frac{\eta_3 {L_{\sigma}^{(\vct{\mu})}}^2}{2}(\sigma_t - \sigma_{t-1})^2 \\
	& \leq \left( -\frac{\eta_3}{2} +\frac{L_{\vct{\mu}} \eta_3^2}{2} \right) \|\bg_t^{\vct{\mu}}\|^2 +  \frac{\eta_3 \eta_2^2 {L_{\sigma}^{(\vct{\mu})}}^2}{2} (g^{\sigma}_t)^2 \\
	& \leq \left( -\frac{\eta_3}{2} +\frac{L_{\vct{\mu}} \eta_3^2}{2} \right) \|\bg_t^{\vct{\mu}}\|^2 +  \frac{\eta_2^2 {L_{\sigma}^{(\vct{\mu})}}^2}{2} (g^{\sigma}_t)^2. \numberthis \label{eq:ae_suff_dec_mu}
\end{align*}

Finally, we have the overall decrease when $\tau$ divides $t$ by summing equation~\eqref{eq:ae_suff_dec_sigma}, \eqref{eq:ae_suff_dec_theta}, and~\eqref{eq:ae_suff_dec_mu},
\begin{align*}
f^{\mathrm{ae}}(\{\vct{\theta}_{i,t}\}, \vct{\mu}_{t}, \sigma_{t}) - f^{\mathrm{ae}}(\{\vct{\theta}_{i,t}\}, \vct{\mu}_{t-1}, \sigma_{t-1})	
&\leq \left( -\eta_1+ \eta_1^2\frac{L_{\vct{\theta}}}{2} \right) \left( \frac{1}{m} \sum_{i=1}^m \left\lVert \nabla_{\vct{\theta}_{i,t-1}} f_i^{\mathrm{ae}}(\vct{\theta}_{i,t-1}, \vct{\mu}_{t-1}, \sigma_{t-1}) \right\rVert^2 \right) \\
 & \quad + \left( -\eta_2+\eta_2^2 \frac{L_\sigma + {L_{\sigma}^{(\vct{\mu})}}^2}{2} \right) (g^{\sigma}_t)^2 + \left( -\frac{\eta_3}{2} +\frac{L_{\vct{\mu}} \eta_3^2}{2} \right) \|\bg_t^{\vct{\mu}}\|^2. 
\end{align*}

\paragraph{When $\tau$ does not divide $t$}

At time steps that are not communication rounds we simply have decrease due to the updates of $\{\vct{\theta}_i\}$, that is,
\begin{align*}
f^{\mathrm{ae}}(\{\vct{\theta}_{i,t}\}, \vct{\mu}_t, \sigma_t) - f^{\mathrm{ae}}(\{\vct{\theta}_{i,t-1}\}, \vct{\mu}_{t-1}, \sigma_{t-1})
&= f^{\mathrm{ae}}(\{\vct{\theta}_{i,t}\}, \vct{\mu}_{t-1}, \sigma_{t-1}) - f^{\mathrm{ae}}(\{\vct{\theta}_{i,t-1}\}, \vct{\mu}_{t-1}, \sigma_{t-1}) \\
&= \frac{1}{m} \sum_{i=1}^m(f_i^{\mathrm{ae}}(\vct{\theta}_{i,t}, \vct{\mu}_{t-1}, \sigma_{t-1}) - f_i^{\mathrm{ae}}(\vct{\theta}_{i,t-1}, \vct{\mu}_{t-1}, \sigma_{t-1})) \\
&\leq \left( -\eta_1+ \eta_1^2\frac{L_{\vct{\theta}}}{2} \right) \left( \frac{1}{m} \sum_{i=1}^m \left\lVert \nabla_{\vct{\theta}_{i,t-1}} f_i^{\mathrm{ae}}(\vct{\theta}_{i,t-1}, \vct{\mu}_{t-1}, \sigma_{t-1}) \right\rVert^2 \right).
\end{align*}

\paragraph{The final bound}
By choosing, $\eta_1 = \frac{1}{L_{\vct{\theta}}}, \eta_2 = \frac{1}{L_{\sigma}+{L_{\sigma}^{(\vct{\mu})}}^2}, \eta_3= \min\{1,\frac{1}{L_{\vct{\mu}}}\}$, and by averaging over time steps while combining two type of decrease, we obtain

\begin{align*}
\frac{1}{T} \sum_{t=1}^T \left( \frac{1}{m} \sum_{i=1}^m \left\lVert \nabla_{\vct{\theta}_{i,t-1}} f_i^{\mathrm{ae}}(\vct{\theta}_{i,t-1}, \vct{\mu}_{t-1}, \sigma_{t-1}) \right\rVert^2 \right)
+\frac{1}{T}  \sum_{\substack{t=1\\t\%\tau=0}}^{T} \|\bg_t^{\vct{\mu}}\|^2
+ \frac{1}{T} \sum_{\substack{t=1\\t\%\tau=0}}^{T} (g^{\sigma}_t)^2
\leq \frac{\max\{L_{\vct{\theta}},L_{\sigma}+{L_{\sigma}^{(\vct{\mu})}}^2,L_{\vct{\mu}},1 \}\Delta_T}{T},
\end{align*}
where $\Delta_T = f^{\mathrm{ae}}(\{\vct{\theta}_{i,0}\}, \vct{\mu}_0, \sigma_0) - f^{\mathrm{ae}}(\{\vct{\theta}_{i,T}\}, \vct{\mu}_{T}, \sigma_{T})$. Given that $\tau$ is a finite constant let us denote $R = T/\tau$ as the number of communication rounds. Then we have, 

\begin{align*}
\frac{1}{T} \sum_{t=1}^T \left( \frac{1}{m} \sum_{i=1}^m \left\lVert \bg_{i,t}^{\vct{\theta}} \right\rVert^2 \right)
	& + \frac{1}{T}  \sum_{\substack{t=1\\t\%\tau=0}}^{T} \|\bg_t^{\vct{\mu}}\|^2
	+ \frac{1}{T} \sum_{\substack{t=1\\t\%\tau=0}}^{T} (g_t^{\sigma})^2 \\
	& = \frac{R}{T} \left( \frac{1}{R} \sum_{t=1}^T \left( \frac{1}{m} \sum_{i=1}^m \left\lVert  \bg_{i,t}^{\vct{\theta}} \right\rVert^2 \right)
	+\frac{1}{R}  \sum_{\substack{t=1\\t\%\tau=0}}^{T} \|\bg_t^{\vct{\mu}}\|^2
	+ \frac{1}{R} \sum_{\substack{t=1\\t\%\tau=0}}^{T} (g^{\sigma}_t)^2 \right) \\
	& \geq \frac{R}{T} \left( \frac{1}{R} \sum_{\substack{t=1\\t\%\tau=0}}^{T} \left( \frac{1}{m} \sum_{i=1}^m \left\lVert  \bg_{i,t}^{\vct{\theta}} \right\rVert^2 \right)
	+\frac{1}{R}  \sum_{\substack{t=1\\t\%\tau=0}}^{T} \|\bg_t^{\vct{\mu}}\|^2
	+ \frac{1}{R} \sum_{\substack{t=1\\t\%\tau=0}}^{T} (g^{\sigma}_t)^2 \right) \\
	& \geq \frac{R}{T} \min_{ \substack{t\in [T] ,\tau\mid t}} \left\{ \left\lVert  \bg_{i,t}^{\vct{\theta}} \right\rVert^2 + \|\bg_t^{\vct{\mu}}\|^2 + (g^{\sigma}_t)^2 \right\}.
\end{align*}
Finally this yields, 
\begin{align*}
	\min_{ \substack{t\in [T] ,\tau\mid t}} \left\{ \left\lVert  \bg_{i,t}^{\vct{\theta}} \right\rVert^2 + \|\bg_t^{\vct{\mu}}\|^2 + (g^{\sigma}_t)^2 \right\} \leq \frac{\max\{L_{\vct{\theta}},L_{\sigma}+{L_{\sigma}^{(\vct{\mu})}}^2,L_{\vct{\mu}},1 \}\Delta^{\mathrm{ae}}_T}{R},
\end{align*}
where $\Delta^{\mathrm{ae}}_T = f^{\mathrm{ae}}(\{\vct{\theta}_{i,0}\}_i, \vct{\mu}_0, \sigma_0) - f^{\mathrm{ae}}(\{\vct{\theta}_{i,T}\}_i, \vct{\mu}_T, \sigma_T)$.

\end{proof}

}
% \section{Experimental Details} \label{app:experiments}
% \input{appendix_experiments}

\end{document}